\documentclass{sig-alternate}

\usepackage{url, bm, color}
\usepackage{subfigure}
\usepackage[colorlinks]{hyperref}

\usepackage{amsfonts,amsmath,amssymb,amsthm,bm}
\usepackage{algorithm}
\usepackage[noend]{algpseudocode}

\usepackage[numbers]{natbib}

\usepackage{multibib}
\newcites{app}{References}

\usepackage[english]{babel}
\usepackage{mdwlist}

\usepackage{algorithm}

\usepackage{comment}

\makeatletter
\def\BState{\State\hskip-\ALG@thistlm}
\makeatother

\usepackage{xspace}

\usepackage{mathtools}

\usepackage{enumitem}
\usepackage{appendix}

\usepackage{multirow}
\newcommand{\eq}[1]{(\ref{#1})}

\newcommand{\inner}[2]{\left\langle #1,#2 \right\rangle}
\newcommand{\rbr}[1]{\left(#1\right)}

\newcommand{\cbr}[1]{\left\{#1\right\}}

\newcommand{\RR}{\mathbb{R}}
\newcommand{\NN}{\mathbb{N}}

\newcommand{\one}{\mathbf{1}}

\newcommand{\vct}{\boldsymbol }

\newcommand{\R}{\mathbb{R}}

\newtheorem{theorem}{Theorem}

\newtheorem{definition}{Definition}

\def\P{\mathbb{P}}

\def\R{\mathbb{R}}
\def\cA{\mathcal{A}}

\def\cN{\mathcal{N}}

\def\cX{\mathcal{X}}

\def\x{\mathbf{x}}

\begin{document}
%

\title{Fast Differentially Private Matrix Factorization}
%
%
%
%
%

\numberofauthors{3} 
%
\author{
%
%
\alignauthor
Ziqi Liu\\
       \affaddr{MOEKLINNS Lab, Computer Science}\\
       \affaddr{Xi'an Jiaotong University}\\
       \affaddr{Xi'an, China}\\
       \email{ziqilau@gmail.com}
\alignauthor
Yu-Xiang Wang\\
       \affaddr{Machine Learning Department}\\
       \affaddr{Carnegie Mellon University}\\
       \affaddr{Pittsburgh, PA}\\
       \email{yuxiangw@cs.cmu.edu}
\alignauthor Alexander J. Smola\\
       \affaddr{CMU Machine Learning}\\
       \affaddr{Marianas Labs Inc.}\\
       \affaddr{Pittsburgh, PA}\\
       \email{alex@smola.org}
}

\maketitle
\begin{abstract}
  Differentially private collaborative filtering is a challenging
  task, both in terms of accuracy and speed. We present a simple
  algorithm that is provably differentially private, while offering
  good performance, using a novel connection of differential privacy
  to Bayesian posterior sampling via Stochastic Gradient Langevin
  Dynamics.  Due to its simplicity the algorithm lends itself to
  efficient implementation. By careful systems design and by
  exploiting the power law behavior of the data to maximize CPU cache
  bandwidth we are able to generate 1024 dimensional models at a rate
  of 8.5 million recommendations per second on a single PC.
\end{abstract}


\section{Introduction}

Privacy protection in recommender systems is a notoriously
challenging problem. 
There are often two
competing goals at stake: similar users are likely to prefer similar
products, movies, or locations, hence sharing of preferences between
users is desirable. Yet, at the same time, this exacerbates the type
of privacy sensitive queries, simply since we are now not looking for
aggregate properties from a dataset (such as a classifier) but for
properties and behavior of other users `just like' this specific
user. Such highly individualized behavioral patterns are shown to facilitate provably effective user de-anonymization \citep{narayanan2008robust,zhang2012guess}.

Consider the case of a couple, both
using the same location recommendation service. Since both spouses
share much of the same location history, it is likely that they will
receive similar recommendations, based on other users' preferences
similar to theirs. In this context sharing of information is
desirable, as it improves overall recommendation quality.

Moreover, since their location history is likely to be very similar,
each of them will also receive recommendations to visit the place that
their spouse visited (e.g.\ including places of ill repute),
regardless of whether the latter would like to share this information
or not. This creates considerable tension in trying to satisfy those
two conflicting goals.

Differential privacy offers tools to overcome these problems. Loosely
speaking, it offers the participants \emph{plausible deniability} in terms of
the estimate. That is, it provides guarantees that the recommendation
would also have been issued with sufficiently high probability if
another specific participant had not taken this action before. This is
precisely the type of guarantee suitable to allay the concerns in the
above situation \citep{Dwork2013}.

Recent work, e.g.\ by \citet{Mcsherry2009} has focused on designing
\emph{custom built} tools for differential private
recommendation. Many of the design decisions in this context are
hand engineered, and it is nontrivial to separate the choices
made to obtain a differentially private system from those made to
obtain a system that works well. Furthermore, none of these systems
\cite{Mcsherry2009, Xin2014} lead to very fast implementations.

In this paper we show that a large family of recommender systems, namely
those using matrix factorization, are well suited to differential
privacy. More specifically, we exploit the fact that sampling from the
posterior distribution of a Bayesian model, e.g.\ via Stochastic
Gradient Langevin Dynamics (SGLD) \citep{Welling2011}, can lead to estimates
that are sufficiently differentially private \citep{Wang2015}. At the same time, their
stochastic nature makes them well amenable to efficient
implementation. Their generality means that we \emph{need not custom-design
a statistical model for differential privacy} but rather that is
possible to \emph{retrofit an existing model} to satisfy these
constraints. The practical importance of this fact cannot be
overstated --- it means that no costly re-engineering of deployed
statistical models is needed. Instead, one can simply reuse the existing inference algorithm with a trivial modification to obtain a
differentially private model.

This leaves the issue to performance. Some of the best reported results
are those using GraphChi \cite{Kyrola2012}, which show that state-of-the-art recommender systems can be built using just a single PC within a
matter of hours, rather than requiring hundreds of computers. In this paper, we show
that by efficiently exploiting the power law properties inherent in the data (e.g.\
most movies are hardly ever reviewed on Netflix), one
can obtain models that achieve peak numerical performance for
recommendation. More to the point, they are 3 times faster than
GraphChi on identical hardware.

In summary, this paper describes the by far the fastest matrix factorization based recommender system and it can be made differentially privately using SGLD without losing performance.
Most competing approaches excel at no more than one of those aspects. Specifically,
\begin{enumerate}[leftmargin=0.4cm,topsep=0pt,itemsep=-1ex,partopsep=1ex,parsep=1ex]
\item It is efficient at the state of the art relative to other
matrix factorization systems.
\begin{itemize}[leftmargin=0.2cm,topsep=0pt]
\item We develop a cache efficient matrix factorization framework for general SGD updates.
\item We develop a fast SGLD sampling algorithm with bookkeeping to avoid adding the Gaussian noise to the whole parameter space at each updates while still maintaining the correctness of the algorithm.
\end{itemize}
\item And it is differentially private.
\begin{itemize}[leftmargin=0.2cm,topsep=0pt]
\item We show that sampling from a scaled posterior distribution for matrix factorization system can guarantee user-level differential privacy.
\item We present a personalized differentially private method for calibrating each user's privacy and accuracy.
\item We only privately release $V$ to public, and design a local recommender system for each user.
\end{itemize}
\end{enumerate}
Experiments confirm that the algorithm can be implemented with high
efficiency, while
offering very favorable privacy-accuracy tradeoff that nearly matches systems without
differential privacy at meaningful privacy level.
\section{Background}
\label{sec:background}

We begin with an overview of the relevant ingredients, namely
collaborative filtering using matrix factorization, differential
privacy and a primer in computer architecture. All three are relevant
to the understanding of our approach. In particular, some basic
understanding of the cache hierarchy in microprocessors is useful for
efficient implementations.

\subsection{Collaborative Filtering}

In collaborative filtering we assume that we have a set of
$\mathcal{U}$ users, rating $\mathcal{V}$ items. We only observe a
small number of entries $r_{ij}$ in the rating matrix $R$. Here
$r_{ij}$ means that user $i$ rated item $j$. A popular tool
\cite{Koren2009} to deal with inferring entries in
$R \in \RR^{|\mathcal{U}| \times |\mathcal{V}|}$ is to approximate $R$
by a low rank factorization, i.e.\
\begin{align}
  R \approx U V^\top
  \text{ where } U \in \RR^{|\mathcal{U}| k}
  \text{ and } V \in \RR^{|\mathcal{V}| k}
\end{align}
for some $k \in \NN$, which denotes the dimensionality of the feature
space corresponding to each item and movie. In other words,
(user,item) interactions are modeled via
\begin{align}
  r_{ij} \approx \inner{u_i}{v_j} + b_i^u + b_j^m + b_0.
\end{align}
Here $u_i$ and $v_j$ denote row-vectors of $U$ and $V$ respectively,
and $b_i^u$ and $b_j^m$ are scalar offsets responsible for a specific
user or movie respectively. Finally, $b_0$ is a common bias.

A popular interpretation is that for a given item $j$, the elements of
$v_j$ measure the extent to which the item possesses those
attributes. For a given user $i$ the elements of $u_i$ measure the extent
of interest that the user has in items that score highly in the corresponding
factors. Due to the conditions proposed in the Netflix contest, it is
common to aim to minimize the mean squared error of deviations between
true ratings and estimates. To address overfitting, a norm penalty is
commonly imposed on $U$ and $V$. This yields the following
optimization problem
\begin{align}
\nonumber
\min_{u,v} \sum_{{i,j} \in R} (r_{ij} - \inner{u_i}{v_j} - b_i^u -
  b_j^m - b_0)^2+ \lambda(||U||_2^2 + ||V||_2^2)
\end{align}
A large number of extensions have been proposed for this model. For
instance, incorporating co-rating information
\cite{Salakhutdinov2007}, neighborhoods, or temporal dynamics
\cite{Koren2010} can lead to improved performance. Since we are
primarily interested in demonstrating the efficacy of differential
privacy and the interaction with efficient systems design, we focus on
the simple inner-product model with bias.

\noindent
{\bfseries Bayesian View.}
Note that the above optimization problem can be viewed as an
instance of a Maximum-a-Posteriori estimation problem. That is, one
minimizes
\begin{align}
  \nonumber
  -\log p(U,V|R, \lambda_r, \Lambda_u, \Lambda_v) = -\log \cN(R|\inner{U}{V}, \lambda_r^{-1})\\\nonumber
  -\log \cN(U|0, \Lambda_u^{-1}) -\log \cN(V|0, \Lambda_v^{-1})
\end{align}
where, up to a constant offset
$$-\log p(r_{ij}|u_i, v_j) = \lambda_r (r_{ij} - \inner{u_i}{v_j} - b_i^u -
  b_j^m - b_0)^2$$
and $-\log p(U) = U \Lambda_u U^\top$ and likewise for $V$. In other
words, we assume that the ratings are conditionally normal, given the
inner product $\inner{u_i}{v_j}$, and the factors
$u_i$ and $v_j$ are drawn from a normal distribution. Moreover, one can also
introduce priors for $\lambda_r, \Lambda_u, \Lambda_v$ with a Gamma distribution $\mathcal{G}(\cdot|\alpha, \beta)$.

While this setting is typically just treated as an afterthought of
penalized risk minimization, we will explicitly use this when
designing differentially private algorithms. The rationale for this is
the deep connection between samples from the posterior and
differentially private estimates. We will return to this aspect after
introducing Stochastic Gradient Langevin Dynamics.

\noindent
{\bfseries Stochastic Gradient Descent.}
Minimizing the regularized collaborative filtering objective is
typically achieved by one of two strategies: Alternating Least Squares
(ALS) and stochastic gradient descent (SGD). The advantage of the
former is that the problem is biconvex in $U$ and $V$ respectively,
hence minimizing $U|V$ or $V|U$ are convex. On the other hand, SGD is
typically faster to converge and it also affords much better cache
locality properties. Instead of accessing e.g.\ all reviews for a
given user (or all reviews for a given movie) at once, we only need to
read the appropriate tuples. In SGD each time we update a randomly
chosen rating record by:
\begin{align}
u_i & \gets (1-\eta_t \lambda) u_i + \eta_t v_j
      \rbr{r_{ij}-\inner{u_i}{v_j} - b_i^u - b_j^m - b_0} \nonumber\\
v_j & \gets (1-\eta_t \lambda) v_j + \eta_t u_i
      \rbr{r_{ij}-\inner{u_i}{v_j} - b_i^u - b_j^m - b_0}
\label{eq:sgd}
\end{align}
One problem of SGD is that trivially parallelizing the procedure
requires memory locking and synchronization for each rating, which
could significantly hamper the performance. \cite{Niu2011} shows that
a lock-free scheme can achieve nearly optimal solution when the data
access is sparse. We build on this \emph{statistical} property to
obtain a \emph{fast system} which is suitable for differential
privacy.

\subsection{Differential Privacy}
\label{sec:diffp}

Differential
privacy (DP)~\citep{dwork2006differential,dwork2006calibrating} aims to
provide means to cryptographically protect personal information in the
database, while allowing aggregate-level information to be accurately
extracted. In our context this means that we protect user-specific
sensitive information while using aggregate information to benefit all
users.


Assume the actions of a statistical database are modeled via a randomized algorithm $\mathcal{A}$. Let the space of data be $\mathcal{X}$ and data sets $X, Y \in \mathcal{X}^n$. Define $d(X, Y)$ to be the edit distance or Hamming distance between data set $X$ and $Y$, for instance if $X$ and $Y$ are the same except one data point then we have $d(X, Y) = 1$.
\begin{definition}[Differential Privacy]
\label{def:diffp}
We call a randomized algorithm $\mathcal{A}$ $(\epsilon, \delta)$-differentially private if for all measurable sets $S \subset \text{Range}(\mathcal{A})$ and for all $X,X'\in \cX^n$ such that the hamming distance $d(X,X')=1$,
\begin{align}
\P(\mathcal{A}(X) \in S) \le \mathrm{exp}(\epsilon) \P(\mathcal{A}(X') \in S) + \delta \nonumber
\end{align}
If $\delta = 0$ we say that $\mathcal{A}$ is $\epsilon$-differential private.
\end{definition}
The definition states that if we arbitrarily replace any individual
data point in a database, the output of the algorithm doesn't change
much. The parameter $\epsilon$ in the definition controls the maximum
amount of information gain about an individual person in the database
given the output of the algorithm. When $\epsilon$ is small, it
prevents any forms of linkage attack to individual data record (e.g.,
linkage of Netflix data to IMDB data \citep{narayanan2008robust}). We
refer readers to~\cite{Dwork2013} for detailed interpretations of the
differential privacy in statistical testing, Bayesian inference and
information theory.

An interesting side-effect of this definition in the context of
collaborative filtering is that it also limits the influence of so-called
whales, i.e.\ of users who submit extremely large numbers of
reviews. Their influence is also curtailed, at least under the
assumption of an equal level of differential privacy per user. In
other words, differential privacy confers robustness for collaborative
filtering.

\citet{Wang2015} show that posterior sampling with bounded
log-likelihood is essentially exponential mechanism
\citep{mcsherry2007mechanism} therefore protecting differential
privacy for free (similar observations were made independently in \citep{mir2013differential,dimitrakakis2014robust}). \citet{Wang2015} also suggests a recent line of
works~\cite{Welling2011,Chen2014, Ding2014} that use stochastic
gradient descent for Hybrid Monte Carlo sampling essentially preserve
differential privacy with the same algorithmic procedure. The
consequence for our application is very interesting: if we trust that
the MCMC sampler has converged, i.e.\ if we get a sample that is approximately
drawn from the posterior distribution, then we can use one sample
as the private release. If not, we can calibrate the MCMC procedure
itself to provide differential privacy (typically at the cost of
getting a much poorer solution).

\subsection{Computer Architecture}


A key difference between generic numerical linear algebra, as commonly
used e.g.\ for deep networks or generalized linear models, and the
methods used for recommender systems is the fact that the access
properties regarding users and items are highly nonuniform. This is a
significant advantage, since it allows us to exploit the caching
hierarchy of modern CPUs to benefit from higher bandwidth than what
disks or main memory access would permit.

\begin{table}[tbh]
\begin{small}
\begin{tabular}{|l|r|r|r|}
\hline
Device & Capacity & Bandwidth read & Bandwidth write \\ \hline
Hard Disk & 3TB & 150MB/s & 100MB/s \\
SSD & 256GB & 500MB/s & 350MB/s \\
RAM & 16GB & 14GB/s & 9GB/s \\
L3 Cache & 6MB & 16-44GB/s & 7-30GB/s\\
L1 Cache & 32KB & 74-135GB/s & 44-80GB/s\\\hline
\end{tabular}
\end{small}
\caption{\label{tb:benchmark} Performance (single threaded) on a
  Macbook Pro (2011) using an Intel Core i7 operating at
  2.0 GHz and 160MT/s transfer rate and 2 memory banks. The spread in
  L1 and L3 bandwidth is due to different packet sizes.}
\end{table}

A typical computer architecture consists of a hard disk, solid-state
drive (SSD), random-access memory (RAM) and CPU cache. Many factors
affect the real available bandwidth, such as read and write patterns,
block sizes, etc. 
We measured this for a desktop computer. See
Table~\ref{tb:benchmark} for a quick overview.
A good algorithm design should be pushing the data flow to CPU cache
level and \emph{hide the latency} from SSD or even RAM and amplify the
available bandwidth.



The key strategy in obtaining high throughput collaborative filtering
systems is to obtain peak bandwidth on \emph{each} of the subsystems
by efficient caching. That is, if a movie is frequently reused, it is
desirable to retain it in the CPU cache. This way, we will neither
suffer the high latency (100ns per request) of a random read from
memory, nor will we have to pay for the comparably slower bandwidth of
RAM relative to the CPU cache. 
This intuition is confirmed in the
observed cache miss rates reported in the experiments in
Section~\ref{sec:experiments}.



\section{Differentially Private \\ Matrix
  Factorization}\label{sec:dpmf}

We start by describing the key ideas and algorithmic framework for
differentially private matrix factorization. The method, which
involves preprocessing data and then sampling from a scaled posterior
distribution, is provably differentially private and has profound
statistical implications. Then we will describe a specific Monte Carlo
sampling algorithm: Stochastic Gradient Langevin Dynamics (SGLD) and justify
its use in our setting. We then come up with a novel way to
personalize the privacy protection for individual users. Finally, we
discuss how to develop fast cache-efficient solvers to exploit
bandwidth-limited hardware such that it can be used for general
SGD-style algorithms.

Our differential privacy mechanism relies on a recent observation that
posterior sampling preserves differential privacy, provided that the
log-likelihood of each user is uniformly
bounded~\citep{Wang2015}. This simple yet remarkable result suggests
that sampling from posterior distribution is differentially private
for free to some extent. In our context, the claim is that, if\footnote{For convenience of notation we will omit the biases from the description
below in favor of a slightly more succinct notation.}
$\max_{U,V,R,i}\sum_{{j} \in R_i} (r_{ij} - \langle u_i, v_j \rangle )^2 \leq B$
then the method that outputs a sample from
\begin{align*}
  P(U,V) \propto \exp\rbr{-\hspace{-3mm}\sum_{(i,j) \in R} (r_{ij} - \inner{u_i}{v_j})^2
  + \lambda (\|U\|_F^2 + \|V\|_F^2)}
\end{align*}
preserves $4B$-differential privacy. Moreover, when we want to set the
privacy loss $\epsilon$ to another number, we can easily do this by
simply rescaling the entire expression by $\epsilon/4B$.

The question now is whether
$\underset{U,V,R,i}{\max}\sum_{{j} \in R_i} (r_{ij} - \langle u_i,
v_j \rangle )^2$
is bounded. Since the ratings are bounded between $1\leq r_{ij}\leq 5$ and we can consider a reasonable sublevel set $\{U,V \mid \max_{i,j}{|u_i^Tv_j|}\leq \kappa\}$, we have every summand to be bounded by $(5+\kappa)^2$.
This does not affect the privacy claim as long as $\kappa$ is chosen independent to the data.

$B$ could still be large, if some particular users rated many
movies. This issue is inevitable even if all observed users have
few ratings, since differential privacy also protects users not in
the database. We propose two theoretically-inspired algorithmic
solutions to this problem:
\begin{description}
\item[Trimming:] We may randomly delete ratings for those who rated a lot of movies so that the maximum number of ratings from a single user $\tau$ will not be too much larger than the average number of ratings. This procedure is the underlying gem that allows OptSpace (the very first provable matrix factorization based low-rank matrix completion method) \citep{keshavan2009matrix} to work.
\item[Reweighting:] Alternatively, one can weight each user appropriately so that those who rated many movies will have smaller weight for each rating. \citet{Mcsherry2009} used this reweighting scheme for controlling privacy loss. A similar approach is considered in the study of non-uniform and power-law matrix completion~\citep{meka2009matrix,srebro2010collaborative}, where the weighted trace norm has the same effect as if we reweight the loss-functions.
\end{description}
In addition, these procedures have their practical benefits for the
robustness of the recommendation system, since they prevent any
malicious user from injecting too much impact into the system, see
e.g., \citet{wang2012stability,mobasher2007toward}.  Another
justification of these two procedures is that, if the fully observed
matrix is truly in a low-dimensional subspace, neither of these two
procedures changes the underlying subspace. Therefore, the solutions
should be similar to the non-preprocessed version.

The procedure for differentially private matrix factorization
(DPMF) is summarized in Algorithm~\ref{alg:DPMF}. Note that this is a
\emph{conceptual} sketch (we will discuss an efficient variant thereof
later).
The following theorem guarantees that our procedure is indeed differentially private.
\begin{algorithm}[tbh]
  \caption{Differentially Private Matrix Factorization }\label{alg:DPMF}
  \begin{algorithmic}[1]
  \Require{Partially observed rating matrix $R\in \R^{m\times n}$ with observation mask $\Omega$. $m=\#$ of movies, $n=\#$ of users. Privacy parameter $\epsilon$, a predefined positive parameter $\kappa$ such that $\{U,V \mid  u_i^Tv_j \in [1-\kappa,5+\kappa] \;\forall i,j\}$, rating range $[1,5]$, max allowable number of ratings per-user $\tau$, number of ratings of each user $\{m_1,...,m_n\}$, weight of each user $w$, tuning parameter $\lambda$.}
      \State $B\gets \max_{i=1,...,n}\min\{\tau,m_i\}
      w_i(5-1+\kappa)^2$.
      \Comment{Compute uniform upper bound.}
      \State Trim all users with ratings $>\tau$.
      \State $F(U,V) := \underset{i\in[i],j\in\Omega_i}{\sum}w_i(R_{ij}- u_i^Tv_j)^2 + \lambda (\|U\|_F^2+\|V\|_F^2)$.
      \State Sample $(U,V) \sim P(U,V)\propto e^{ -\frac{\epsilon}{4B} F(U,V) }$
      \While{$u_i^Tv_j \notin [1-\kappa,5+\kappa] \;\text{for some } i,j$}
      \State Sample $(U,V) \sim P(U,V)\propto  e^{ -\frac{\epsilon}{4B} F(U,V) }$
      \EndWhile\
      \State \textbf{return} $(U,V)$
  \end{algorithmic}
\end{algorithm}
\begin{theorem}\label{thm:privacy}
Algorithm~\ref{alg:DPMF} obeys $\epsilon$-differential privacy if the sample is exact and $(\epsilon,(1+e^\epsilon)\delta)$-differential privacy if the sample is from a distribution $\delta$-away from the target distribution in $L_1$ distance.
\end{theorem}
The proof (given in the appendix), shows that this procedure uses in
fact the exponential mechanism \citep{mcsherry2007mechanism} with
utility function being the negative MF objective and its sensitivity
being $2B$. Note that this can be extended to considerably more
complex models. This is the strength of our approach, namely that a
large variety of algorithms can be adapted quite easily to
differential privacy capable models.


\noindent
{\bfseries Statistical properties.}
How about the utility of this procedure? We argue that we do not lose much accuracy by sampling from the a distribution instead of doing exact optimization. Here we define utility/accuracy to be how well this output predicts for new data.

Our matrix factorization formulation can be treated as a maximum a posteriori (MAP) estimator of the Bayesian Probabilistic Matrix Factorization (BPMF) \citep{Salakhutdinov2008}, therefore, this distribution we are sampling from is actually a scaled-version of the posterior distribution.

When $\epsilon=4B$, \citet{Wang2015} shows that a single sample from
the posterior distribution is consistent whenever the Bayesian model
that gives rise to $f(\theta)$ is consistent and asymptotically only a
factor of $2$ away from matching the Cram\'{e}r-Rao lower bound
whenever the asymptotic normality (Bernstein-Von Mises Theorem) of the
posterior distribution holds. Therefore, we argue that by taking only
one sample from the posterior distribution, our results will not be
much worse than estimating the MAP or the posterior mean estimator in
BPMF. Moreover, since the results do not collapse to a point estimator,
the output from this sampling procedure does not tend to overfit
\citep{Welling2011}.

When $\epsilon<4B$ we will start to lose accuracy, but since we are
still sampling from a scaled posterior distribution, the same
statistical property applies and the result remains
asymptotically near optimal with asymptotic relative efficiency $1+\sqrt{4B/\epsilon}$. In fact, monotonic rescaling of $U$ and $V$ leaves
the relative \emph{order} of ratings unchanged.

\subsection{Personalized Differential Privacy}
Another interesting feature of the proposed procedure is that it allows us to calibrate the level of privacy protection for every user independently, via a novel observation that weights assigned to different users are linear in the amount of privacy we can guarantee for that particular user.


We will use the same sampling algorithm, and our guarantees in
Theorem~\ref{thm:privacy} still hold. The idea here is that we can
customize the system so that we get a lower basic privacy protection
for all users, say $\epsilon=4B$. As we explained earlier this is the
level of privacy that we can get more or less ``for free''. The
protection of DP is sufficiently strong as to include even those
users that are not in the database.

By adjusting the weight parameter, we can make the privacy protection
stronger for particular users according to how much they set they want
privacy. This procedure makes intuitive sense because if some user
wants perfect privacy, we can set their weight to $0$ and they are
effectively not in the database anymore. For people who do not care
about privacy, their ratings will be assigned default
weight. Formally, we define personalized differential privacy as
follows:
\begin{definition}[Personalized Differential Privacy]\label{def:DPpersonal}
	An algorithm $\cA$ is $(\epsilon,\delta)$-personalized differentially private for User $i$ in database $X$ if for any measureable set $S$ in the range of the algorithm $\cA$
	$$
	\P(\cA(X)\in S) \leq  e^{\epsilon}\P(\cA(X')\in S) +\delta.
	$$
	for any $X\in \cX^n$ and $X'$ is either $X\cup \{x_i\}$ or $X\backslash\{x_i\}$.
\end{definition}
We claim that:
\begin{theorem}\label{thm:personal_privacy}
	If we set $w_i$ for User-$i$ such that
	$$B_i:=\min\{\tau,m_i\}w_i(4+\kappa)^2\leq B,$$
	then Algorithm~\ref{alg:DPMF} guarantees $\frac{\epsilon
		B_i}{2B}$-personalized differential privacy for User $i$.
\end{theorem}
The proof is a straigtforward verification of
the definition. We defer it to the Appendix. Note that if we set $\epsilon=4B$ (so we are
essentially sampling from the posterior distribution), we get
$2B_i$-Personalized DP for user $i$.

In summary, if we simply set $\epsilon=4B$, the method protects $4B$-differential privacy for everybody at very little cost and by setting the weight vector $w$, we can provide personalized service for users who demands more stringent DP protection. To the best of our knowledge, this is the first method of its kind to protect differential privacy in a \emph{personalized} fashion.

\section{Efficient sampling via SGLD}\label{sec:sgld}

Clearly, sampling from $\exp\rbr{-\frac{\epsilon}{4B} F(U,V)}$ is
nontrivial. For a tractable approach we use a recent MCMC method named
stochastic gradient Langevin dynamics (SGLD) \citep{Welling2011},
which is an annealing of stochastic gradient descent and Langevin
dynamics that samples from the posterior distribution
\citep{neal2011mcmc}. The basic update rule is
\begin{align}\label{eq:sgld}
(u_i,v_j) = (u_i,v_j) - \eta_t \widehat{\nabla}_{(u_i,v_j)} F(U,V) + \cN(0,\eta_t I)
\end{align}
where $\widehat{\nabla}_{(u_i,v_j)} F(U,V)$ is a stochastic gradient
computed using only one or a small number of ratings. In other words,
the updates are almost identical to those used in stochastic gradient
descent. The key difference is that a small amount of Gaussian noise
is added to the updates. This allows us to solve it extremely
efficiently.  We will describe our efficient implementation of this
algorithm in Section \ref{sec:fsgld}.

The basic idea of SGLD is that when we are far away from the basin of
convergence, the gradient of the log-posterior
$\widehat{\nabla}_{(u_i,v_j)} F(U,V)$ is much larger than
the additional noise so the algorithm behaves like stochastic gradient
descent. As we approach the basin of convergence and $\eta_t$ becomes
small, $\sqrt{\eta_t}\gg \eta_t$ so the noise dominates and it behaves
like a Brownian motion. Moreover, as $\eta_t$ gets small, the
probability of accepting the proposal in Metropolis-Hastings
adjustment converges to $1$, so we do not need to do this adjustment
at all as the algorithm proceeds, as designed above.

This seemingly heuristic procedure was later shown to be consistent in
\citep{sato2014approximation,teh2014consistency}, where asymptotic
``in-law'' and ``almost sure'' convergence of SGLD to the correct
stationary distribution are established. More recently,
\citet{vollmer2015non} further strengthens the convergence guarantee
to include any finite iterations. This line of work justifies our
approach in that if we run SGLD for a large number of iterations, we
will end up sampling from the distribution that provides us
$(\epsilon,\delta)$-differential privacy. By taking more iterations,
we can make $\delta$ arbitrarily small.


\section{System Design}

The performance improvement over existing libraries such as GraphChi
are due to both cache efficient design, prefetching, pipelining, the
fact that we exploit the power law property of the data, and by
judicious optimization of random number generation. This leads to a
system that comfortably surpasses even moderately optimized GPU
codes.

We primarily focus on the Stochastic Gradient Descent solver and
subsequently we provide some details on how to extend this to
SGLD. Inference requires a very large number of following operations on data:
\begin{itemize}
\item Read a rating triple $(i,j,r_{ij})$, possibly from
  disk, unless the data is sufficiently tiny to fit into RAM.
\item For each given pair $(i,j)$ of users and items fetch
  the vectors $u_i$ and $v_j$ from memory.
\item Compute the inner product $\inner{u_i}{v_j}$ on the CPU.
\item Update $u_i, v_j$ and write their new values to RAM.
\end{itemize}
To illustrate the impact of these operations consider training a $2,048$
dimensional model on the $10^8$ rating triples of Netflix. Per
iteration this requires over 3.2TB read/write operations to RAM. At a
main memory bandwidth of 20GB/s and a latency of 100ns for each of the
200 million cache misses each pass would take over 6 minutes. Instead,
our code accomplishes this task in approximately 10 seconds by using
the steps outlined below.

\subsection{Processing Pipeline}

To deal with the dataflow from disk to CPU, we use a pipelined
design, decomposing global and local state akin to
\cite{Ahmed2012}. 
This means that we process users
sequentially, thus reducing the retrieval cost per user, since the
operations are amortized over all of their ratings. This effectively
halves IO. Moreover, since the data cannot be assumed to fit into RAM,
we pipeline reads from disk. This hides latency and avoids stalling
the CPUs. The writer thread periodically snapshots the model, i.e.\
$U$ and $V$ to disk.

Note that for personalized recommender systems that require
considerable personalized hidden state, such as topic models, or
autoregressive processes, we may want to write a snapshot of the
user-specific data, too.

\begin{algorithm}
\caption{Cache efficient Stochastic Gradient Descent}\label{alg:sgd}
\begin{algorithmic}[1]
\Require parameters $U$, $V$; ratings $R$; $P$ threads,
\State {\bfseries preprocessing} Split $R$ into $B$ blocks;
\Procedure{Read}{} \Comment{Keep pipeline filled}
\While {\#blocks in flight $\le P$}
        \State Read: block $b$ from disk
        \State Sync: notify {\sc Update} about $b$
\EndWhile
\EndProcedure
\Procedure{Update}{}  \Comment{Update $U$, $V$}
      \While{at least one of $P$ processors is available}
      \State Sync: receive a new block $b$ from {\sc Read}
      \For {user $i$ in $b$}
      		\For{each rating $r_{ij} \in b$ from user $i$}
			\State Prefetch next movie factor $v_{j+1}$ from data stream
			\State $u_i \gets u_i - \eta_t \widehat{\nabla}_{u_i}$
			\State $v_j \gets v_j - \eta_t
                        \widehat{\nabla}_{v_j}$
                        \State ($\widehat{\nabla}$ is
                          either the exact or private gradient)
		\EndFor
		      \EndFor
      \EndWhile
\EndProcedure
\Procedure{Write}{}
\State {\bfseries if} $B_t$ blocks processed {\bfseries then}
 save $U,V$
\EndProcedure
\end{algorithmic}
\end{algorithm}

\subsection{Cache Efficiency}

The previous reasoning discussed how to keep the data pipeline filled
and how to reduce the user-specific cache misses by preaggregating them
on disk. Next we need to
address cache efficiency with regard to movies.  More to
the point, we need to exploit cache locality relative to the CPU
\emph{core} rather than simply avoiding cache misses.
The basic idea is that each CPU core exactly reads a cache line
(commonly 64 bytes) from RAM each time, so algorithm designers should
not waste it until that piece of cache line is fully utilized.

We exploit the fact that movie ratings follow a power law
\cite{hartstein2008nature}, as is evident e.g.\ on Netflix in
Figure~\ref{fig:powerlaw}. This means that if we succeed at keeping
frequently rated movies in the CPU cache, we should see substantial
speedups. Note that traditional matrix blocking tricks, as widely used
for matrix multiplications operations are not useful, due to the
sparsity of the rating matrix $R$. Instead, we decompose the movies
into tiers of popularity. To illustrate, considering a decomposition into three
blocks consisting of the Top 500, the Next 4000, and the remaining
long tail.

Within each block, we process a batch of users simultaneously. This
way we can preserve the associated user vectors $u_i$ in cache and we
are likely to cache the movie vectors, too (in particular for the Top
500 block). Also, parallelizing all the updates for multiple users
does not require locks. Movie parameters are updated in a Hogwild
fashion \cite{Niu2011}.

This design is particularly efficient for low-dimensional models since
the Top 500 block fits into L1 cache (this amounts to 44\% of all
movie ratings in the Netflix dataset), the Next 4000 fits into L2, and
ratings will typically reside in L3. Even in the extreme case of
$2048$ dimensions we can fit about $55\%$ of all ratings into cache,
albeit L3 cache.

\begin{figure}
\centering
\includegraphics[width=\columnwidth]{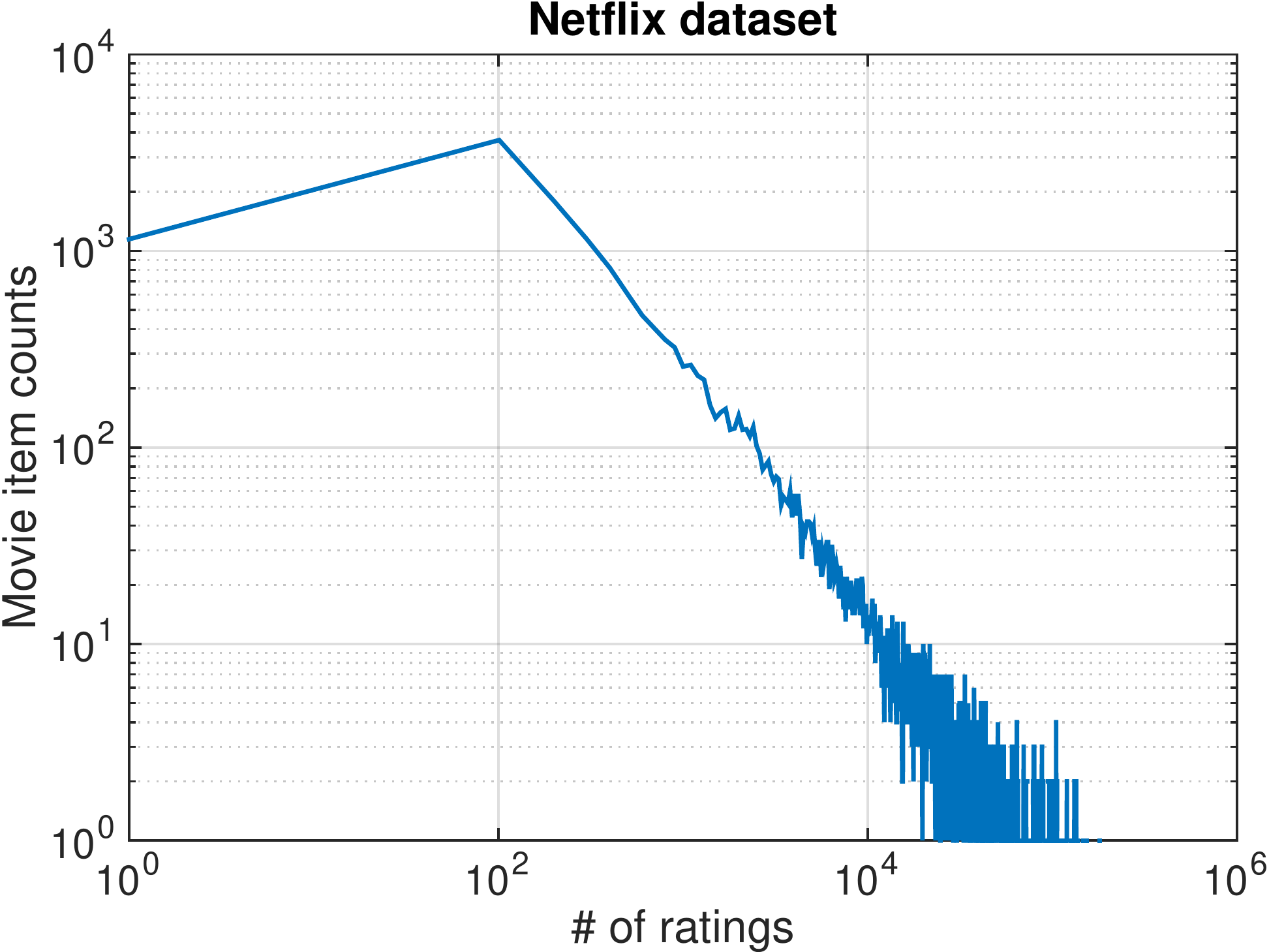}
\includegraphics[width=\columnwidth]{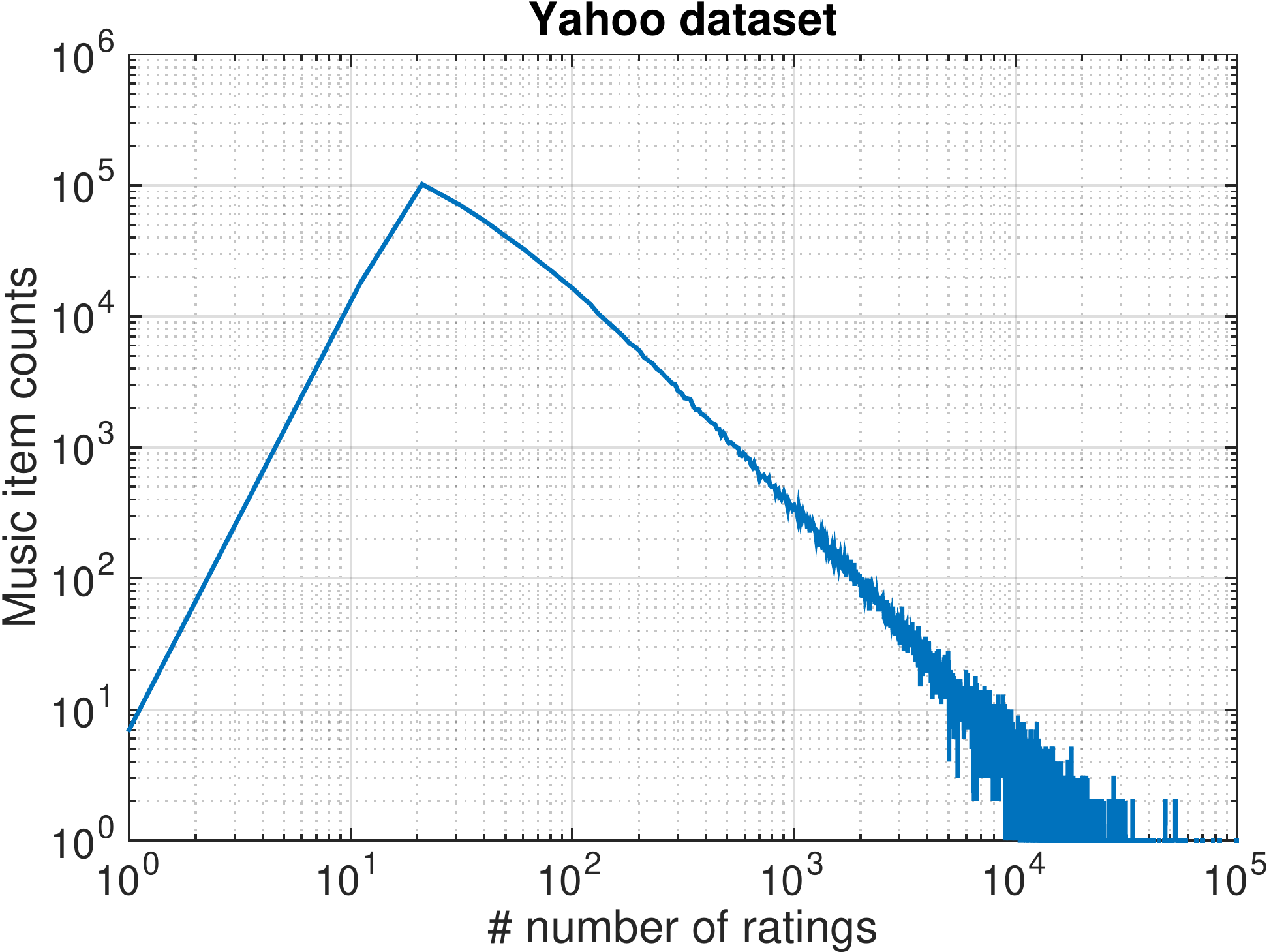}
\caption{Distribution of items (Movies/Music pieces) as a function of their number of ratings. Many movies have 100 ratings or less, while the majority of ratings focuses on a small number of movies.\label{fig:powerlaw}}
\end{figure}

\subsection{Latency Hiding and Prefetching}

To avoid the penalty for random requests we perform latency hiding by
prefetching. That is, we actively request $v_j$ in advance before the
rating $r_{ij}$ is to be updated. For dimensions less than 256,
accurate prefetching leads to a dataflow of $v_j$ into L1
cache. Beyond that, the size of the latent variables could be too big to
benefit from the lowest level of caching due to limited size of caches in modern computers. We provide a detailed
caching analysis in Section~\ref{sec:experiments} to illustrate the
effect of these techniques.

\subsection{Optimizations for SGLD}
\label{sec:fsgld}

The data flow of SGLD is almost analogous to that in SGD, albeit with
a number of complications. First off, note that \eq{eq:sgld} applies
to the whole parameter \emph{matrix} $U,V$ rather than just to a
single vector. Following \cite{Clewett2015} we can derive an unbiased
approximation of $\widehat{\nabla}_{u_i}$ in \eq{eq:sgld} which is
nonzero only for $(u_i, v_j)$ as follows:
$$
\widehat{\nabla}_{u_i} = -N \lambda_r \, (r_{ij} - \inner{u_i}{v_j})\,v_j + \frac{N}{N_i}u_i^\top \Lambda_u u_i
$$
where $N,N_i$ denote number of rating data rated by all and rated by user $i$ respectively. The parameters $\lambda_r, \Lambda_u, \Lambda_v$ do not incur any
major cost --- $\Lambda_u, \Lambda_v$ are diagonal matrices with a
Gamma distribution over them. We simply perform Gibbs sampling once
per round. However, the most time-consuming part is to sample the
remaining vectors, i.e.\ P$(U^{-i},V^{-i}|R,\mathrm{rest})$ since it
both requires dense updates and moreover, it requires many random
numbers, which adds nontrivial cost.  
\begin{description}
\item[Dense Updates:] Note that unless we encounter the triple
  $(i,j,r_{ij})$ all other parameters are only updated by adding
  Gaussian noise. This means that by keeping track of when a parameter
  was last updated, we can simply aggregate the updates (the Normal
  distribution is closed under addition). That is, $c_i$ subsequent
  additions amount to a single draw from $\mathcal{N}(0, c_i \eta)$. The is
  possible since we only need to know the value of $u_i, v_j$ whenever
  we encounter a new triple. 
\item[Table Lookup:] Drawing iid samples from a Gaussian is quite
  costly, easily dominating all other floating point operations
  combined. We address this by pre-generating a large table of
  numbers \cite{Marsaglia2004} and then by performing random lookup
  within the table. More to the point, a lookup table of $r$ random
  numbers is statistically indistinguishable from the truth until we
  draw $O(r^2)$ samples from it (this follows from the slow rate of
  convergence for two-sample tests), hence a few MB of data
  suffice. Finally, for cache efficiency, we read contiguous segments
  with random offset (this adds a small amount of dependence which is
  easily addressed by using a larger table). 
  
  A cautionary note is that the impact of this approach on privacy, namely how it affects the stationary distribution of the SGLD, is unknown. In our experiments, the results are indistinguishable for any moderately sized finite look-up tables (see our experiments in Section~\ref{sec:exp_conv}).
\end{description}

\section{Experiments and Discussion}
\label{sec:experiments}

\begin{figure*}[tbfh]
	\centering
\includegraphics[width=0.46\textwidth,height=0.4\textwidth]{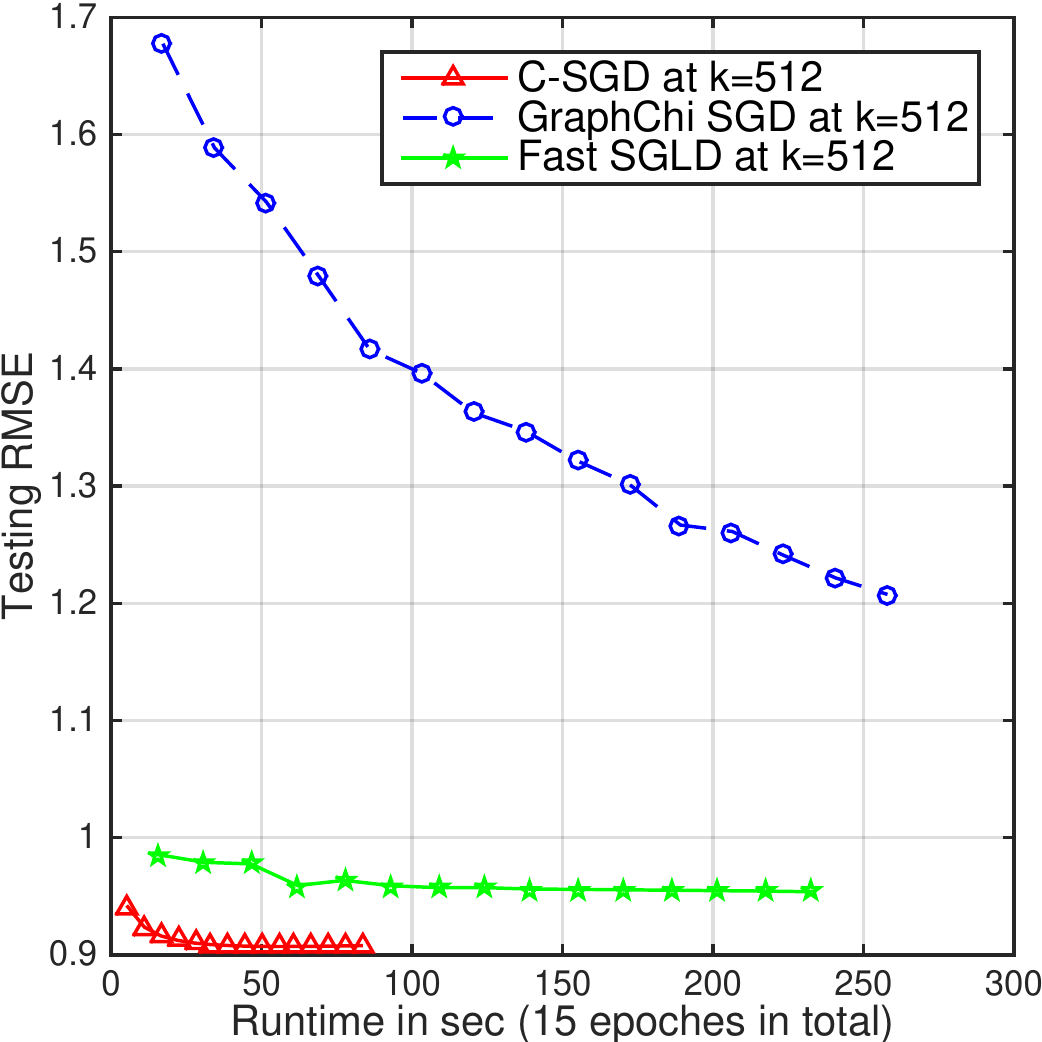}
\includegraphics[width=0.46\textwidth,height=0.4\textwidth]{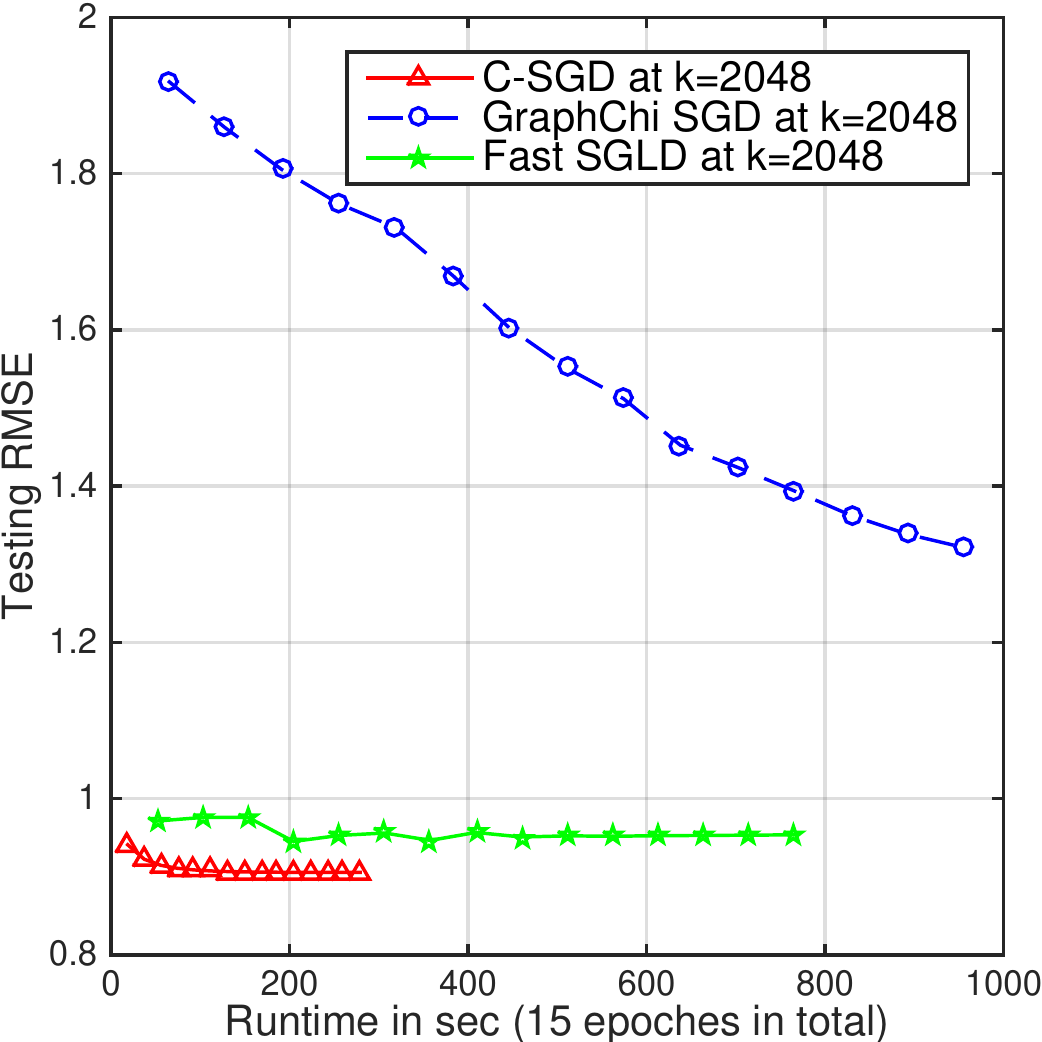}
\includegraphics[width=0.46\textwidth,height=0.4\textwidth]{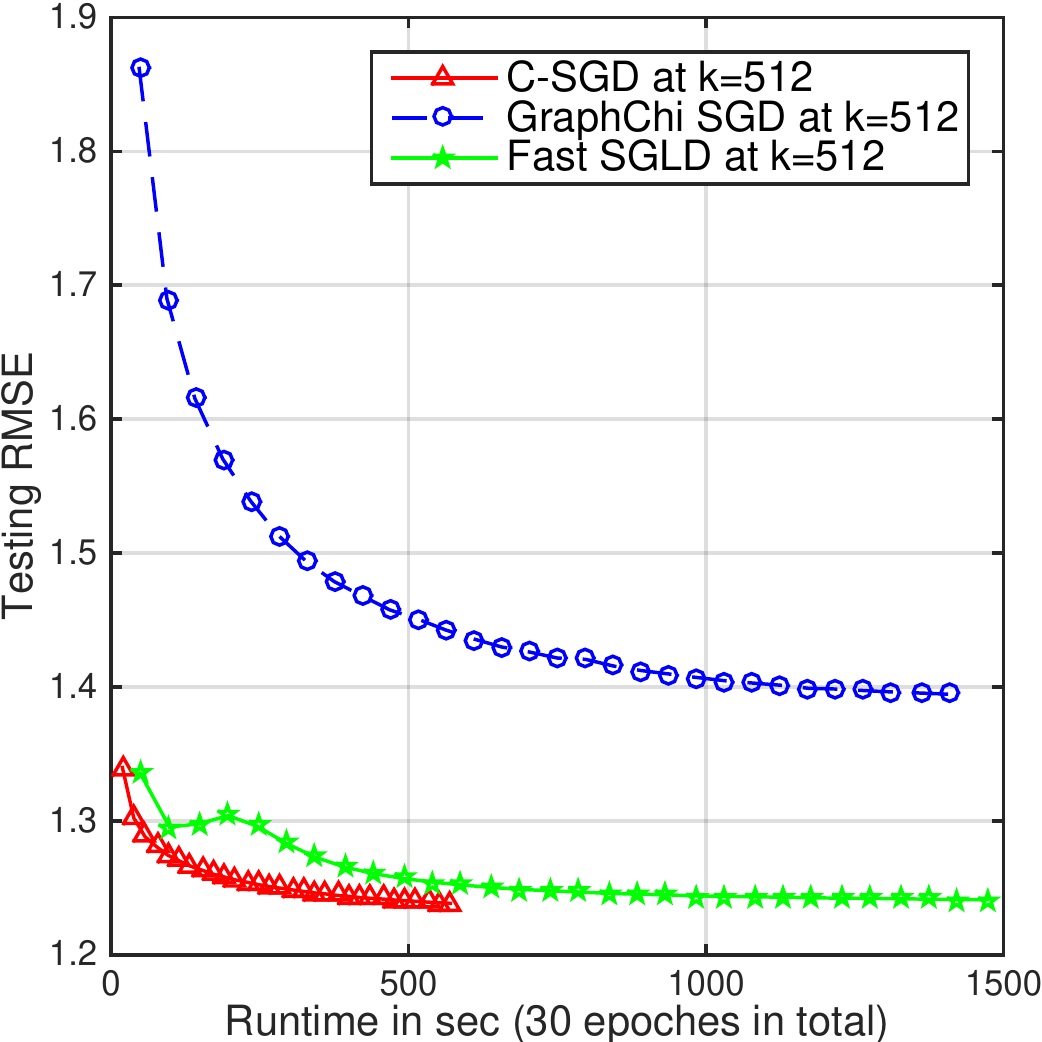}
\includegraphics[width=0.46\textwidth,height=0.4\textwidth]{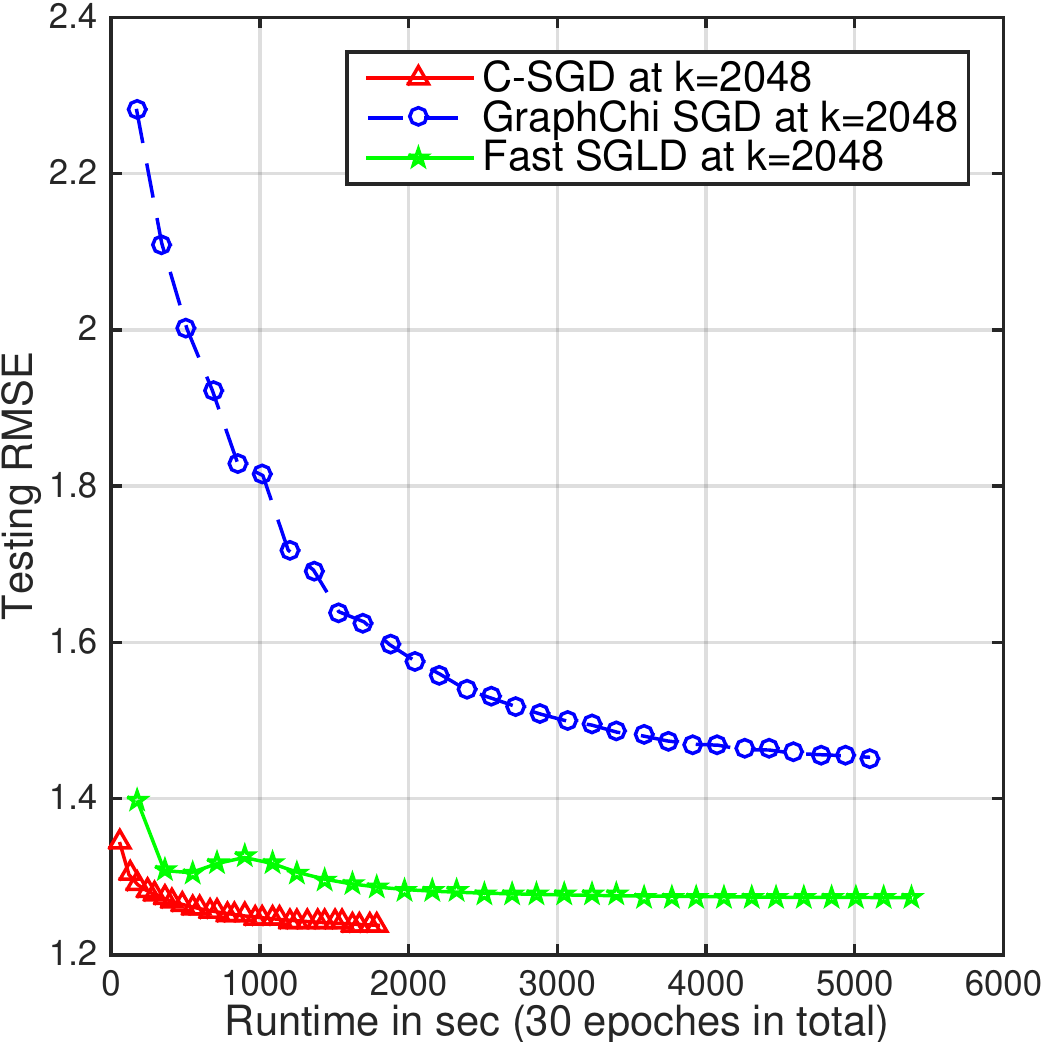}
\caption{Runtime comparisons of the C-SGD solver, differentially private SGLD solver
  vs.\ non-private GraphChi/Graphlab on identical hardware, a Amazon AWS c3.8xlarge instance. Note that regardless of the dimensionality of
  the factors (512, 2048) our C-SGD is
  approximately 2-3 times faster than GraphChi, and differentially private SGLD also can be comparable with Graphchi in very high dimension (Top: Netflix, Bottom: Yahoo).\label{fig:timing}}
\end{figure*}

\begin{figure}[tbhf]
\centering
\includegraphics[width=\columnwidth]{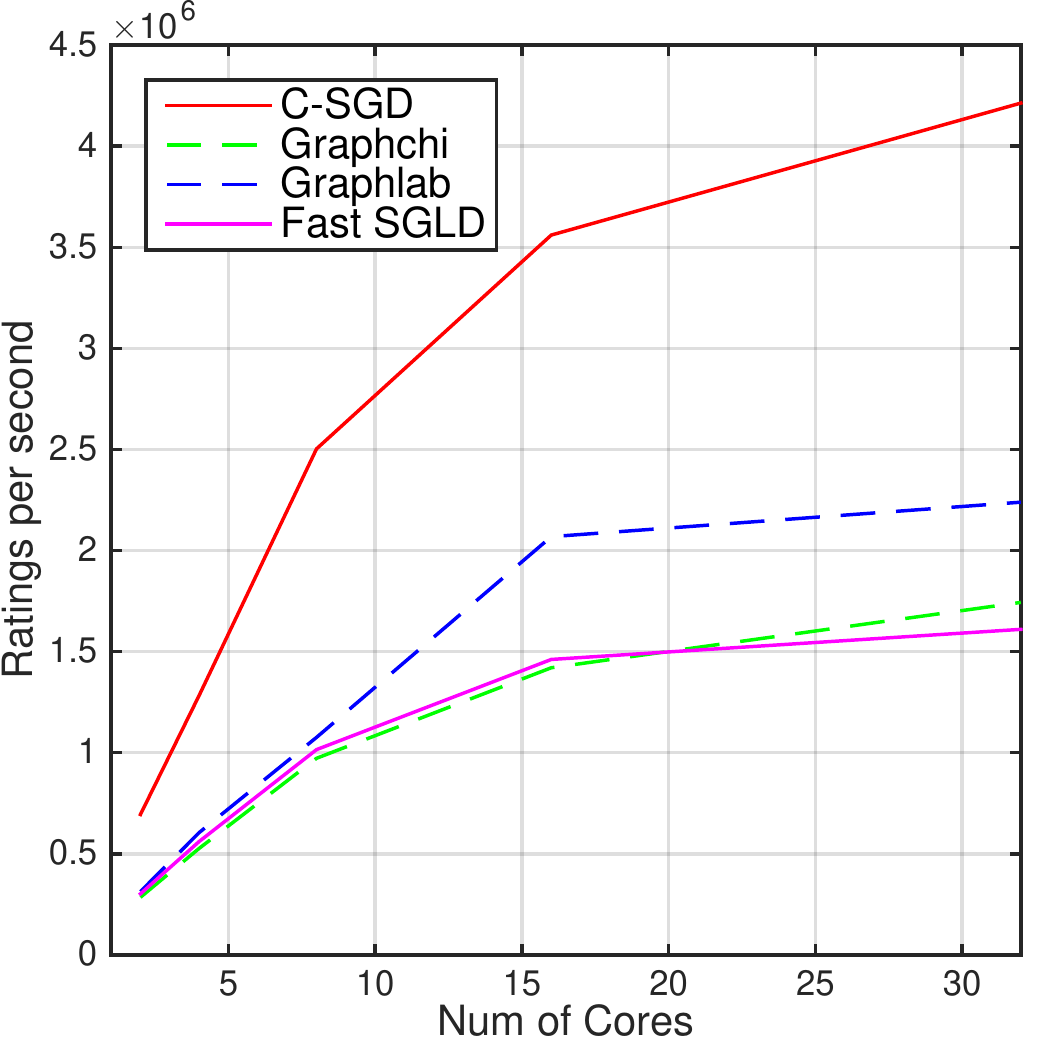}
\caption{Throughput on Yahoo over different number of cores with dimension 2048.\label{fig:through}}
\end{figure}

\begin{figure}[tbhf]
\centering
\includegraphics[width=\columnwidth]{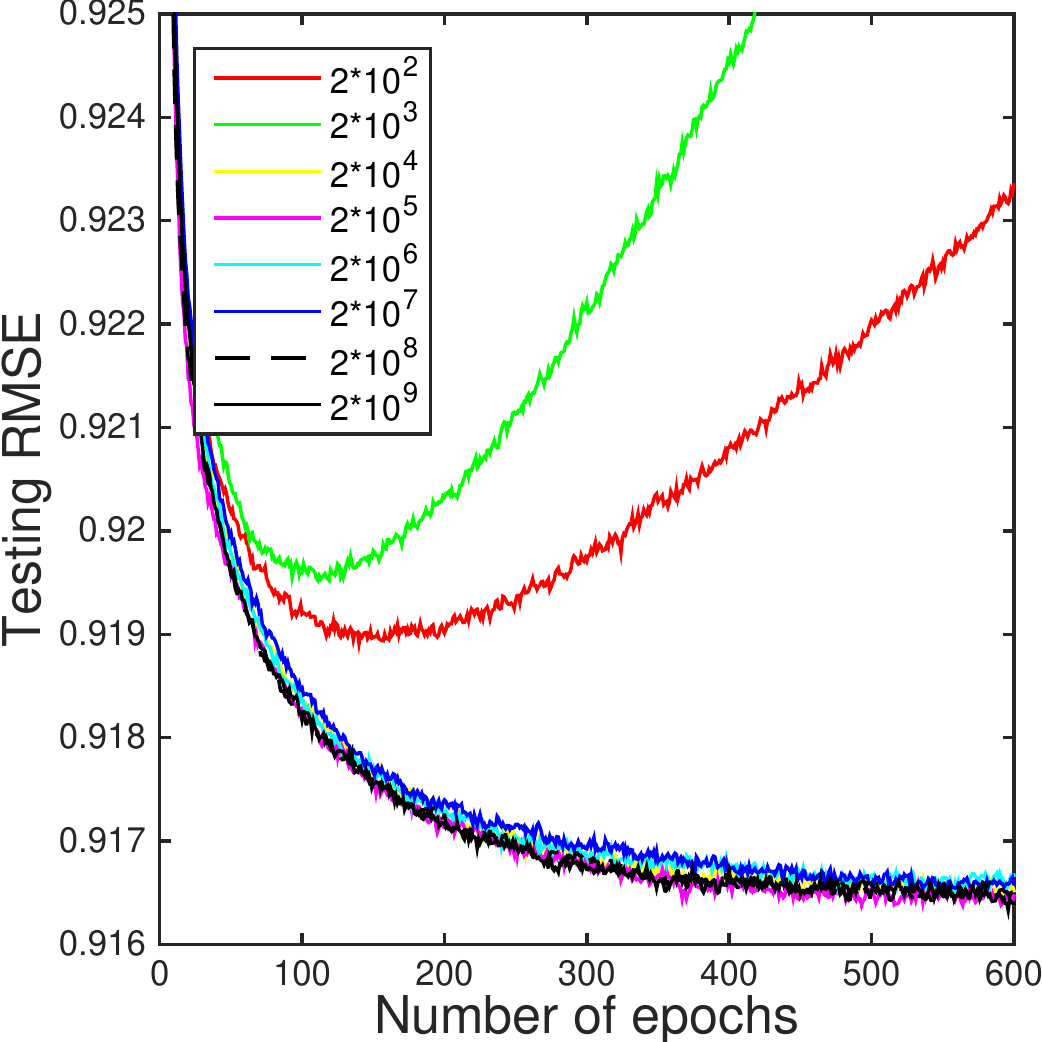}
\caption{Size of Gaussian lookup table vs.\ test RMSE on Netflix data with dimension 16.\label{fig:scale}}
\end{figure}

We now investigate the efficiency and accuracy of our fast SGD solver and Stochastic
Gradient Langevin Dynamics solver, compared with state-of-the-art
available recommenders. We also explore the differentially private accuracy by using our proposed
method while varying different privacy budgets. 


\subsection{Comparisons}

We compare the performance of both the SGD solver and the SGLD solver
to other publicly available recommenders and one closed-source
solver. In particular, we compare to both CPU and GPU solvers, since
the latter tend to excel in massively parallel floating point operations. 
\begin{description}
\item[GraphChi] Most of our experiments focus on a direct comparison to
GraphChi~\cite{Kyrola2012}. This is primarily due the fact that the
code for GraphChi is publicly available as open source and its very
good performance. 
\item[GraphLab Create] is a closed source data analysis platform~\citep{graphlab}. It
  is currently the fastest recommender system available, being
  slightly faster than GraphChi. We compared our system to GraphLab
  Create, albeit without fine-grained diagnostics that were possible
  for GraphChi. 
\item[BidMach] is a GPU based system \cite{BidMach}. It reports
  runtimes of 90, 129 and 600 seconds respectively for 100, 200 and
  500 dimensions using an Amazon g2.2xlarge instance for the Netflix
  dataset.\footnote{\url{http://github.com/BIDData/BIDMach/wiki/Benchmarks}} This is slower than the runtimes of 48, 63, and 83 seconds
  for 128, 256, and 512 that we achieve without GPU optimization on a
  c3.8xlarge instance. 
\item[Spark] is a distributed system (Spark MLlib) for inferring
  recommendations and factorization. In recent comparison the argument
  has been made that it is somewhat
  slower\footnote{\url{http://stanford.edu/~rezab/sparkworkshop/slides/xiangrui.pdf},
    Slide 31} than GraphLab while being substantially faster than
  Mahout. 
\end{description}
\smallskip
\subsection{Data}

We use two datasets --- the well known Netflix Prize dataset,
consisting of a training set of 99M ratings spanning 480k customers
and their ratings on almost 18k, each movie being rated at a scale of
$1$ to $5$ stars. Additionally, we use their released validation set
which consists of $1.4$M ratings for validation purposes. 

Secondly, we use the Yahoo music recommender dataset, consisting of
almost 263M ratings of 635k music items by 1M users. We also use the released validation set which consists of 6M ratings for validation. We re-scale each rating at a scale of $0$ to $5$. We compare
performance on both datasets since their sampling strategies are
somewhat incomparable (e.g.\ Netflix has considerable covariate shift
in the test dataset). Moreover, this larger dataset poses further
challenges on the cache efficiency due to the larger number of items
to be recommended.

\subsection{Runtime}


For efficient computation, GraphChi first needs to preprocess data
into shards by the proposed parallel sliding
windows~\cite{Kyrola2012}. Once the data is partitioned, it can
process the graphs efficiently.
For comparison, we partition both rating matrix of Netflix prize data 
and Yahoo Music data into blocks with each block contains all the ratings 
come from around 1000 users. Each time our algorithms read one block from disk. 
For Graphchi and Graphlab Create we use the default partition strategy.
We run all the experiments on an Amazon c3.8xlarge instance running
Ubuntu 14.04 with 32 CPUs and 60GB RAM.

For SGD-based methods We initialize the initial learning rate and regularizer $\eta_0=0.02, \lambda=5\cdot 10^{-3}$ for Netflix data, and $\eta_0=\{0.1,0.08,0.06\}, \lambda=5\cdot 10^{-2}$ for Yahoo Music data. We update learning rate per round as $\eta_t = \eta_{0}/t^{\gamma}$. We also use the same decay rate $\gamma=1$ for both dataset. For our fast SGLD solver, we set $\eta_0=\{2\cdot10^{-10}, 1\cdot 10^{-10}, 9\cdot 10^{-11}\}$ and hyperparameters $\alpha=1.0, \beta=100.0$. And we set decay rate $\gamma=0.6$ for Netflix data and $\gamma=\{0.8,0.9\}$ for Yahoo data. In practice to speed up SGLD's burn-in procedure, we multiply learning rate by a temperature parameter $\zeta$~\cite{Chen2014} in the Gaussian noise $\cN(0, \zeta \cdot \eta_t)$ with $\sqrt{\zeta \cdot \eta_t}\gg \eta_t$. We set $\zeta = \{0.07, 0.9\}$ for Netflix data and Yahoo data.

Since it is nontrivial to observe the test RMSE error in each epoch
when using Graphlab Create, we only report the timing of Graphlab
Create and all other methods in
Figure~\ref{fig:timing_dimension}. Note that we were unable to obtain
performance results from BidMach for the Yahoo dataset, since Scala
encountered memory management issues. However, we have no reason to
believe that the results would be in any way more favorable to BidMach
than the findings on the Netflix dataset. For reproducibility the
results were carried out on an AWS {\tt g2.8xlarge} instance. 

To illustrate the convergence over time. We run all the methods in a fixed number of epochs. 
That is 15 epochs and 30 epochs respectively because we observe that our SGD solver can reach the convergence at that time. 
Figure~\ref{fig:timing} shows our timing results along with
convergence while we vary dimensions of the models. 

Both of our solvers, i.e.\ C-SGD and Fast SGLD benefit from our caching algorithm.
C-SGD is around 2 to 3 times faster than GraphChi and Graphlab while
simultaneously outperforming the accuracy of GraphChi. The primary
reason for the discrepancy in performance can be found in the order in
which GraphChi processes data: it partitions data (bother users and
items) into random subsets and then optimizes only over one such
subblock at a time. While the latter is fast, it negatively affects
convergence, as can be seen in Figure~\ref{fig:timing}. 

Note that the algorithm required for Fast SGLD is rather more complex,
since it performs sampling from the Bayesian posterior. Consequently,
it is slower than plain SGD. Nonetheless, its speed is comparable to
GraphChi in terms of throughput
(despite the latter solving a much simpler problem). 
One problem of SGLD is that
the more complex the models are, the worse its convergence becomes, due to
the fact that we are sampling from a large state space. This is possibly due to the
slow mixing of SGLD, which is a known problem of SGLD
\citep{ahn2012bayesian}. Improving the mixing rate by considering a
more advanced stochastic differential equation based sampler,
e.g.\ \citep{Chen2014,Ding2014}, while keeping the cache efficiency
during the updates will be important future work. To our best knowledge
we are the first to report the convergence results of SGLD at this
scale.




\begin{figure}[tbhf]
\centering
\includegraphics[width=\columnwidth]{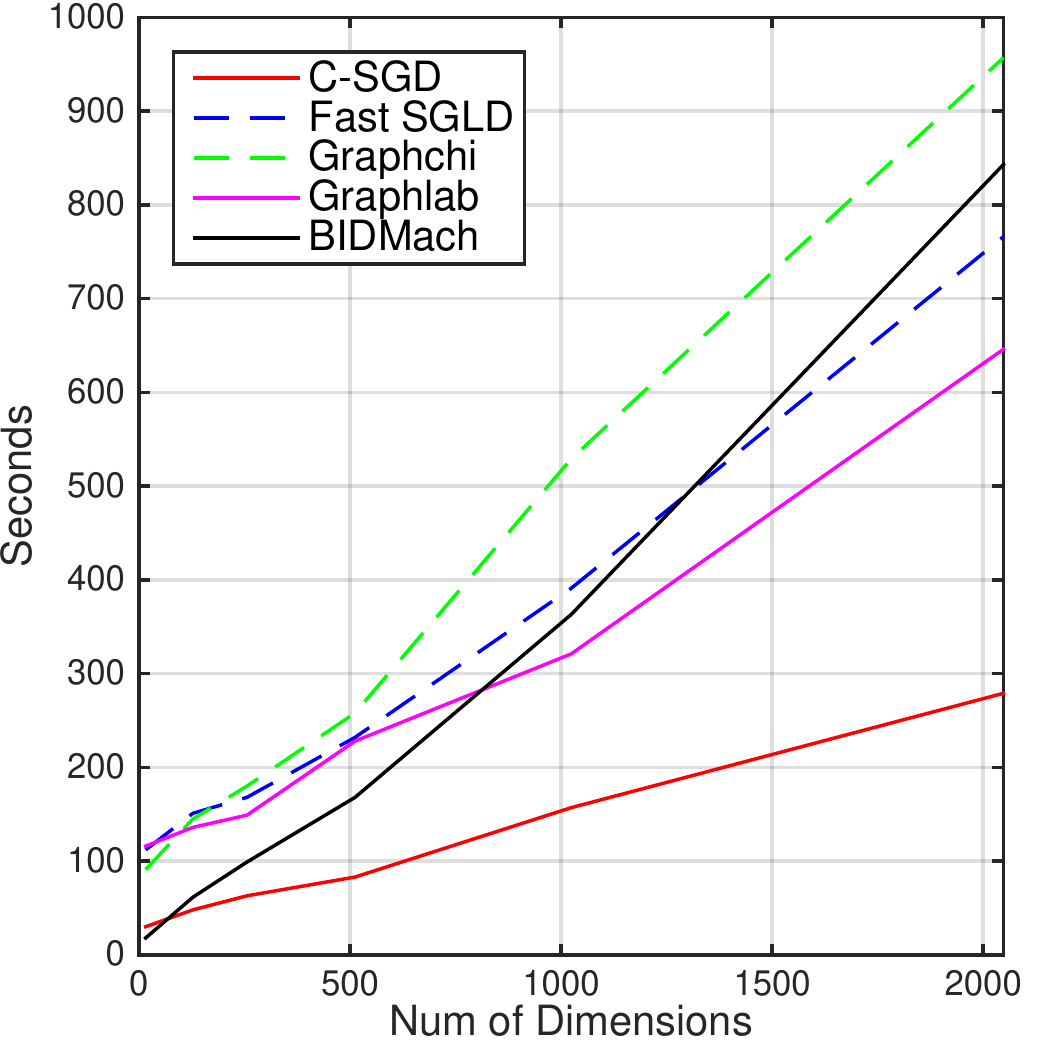}
\includegraphics[width=\columnwidth]{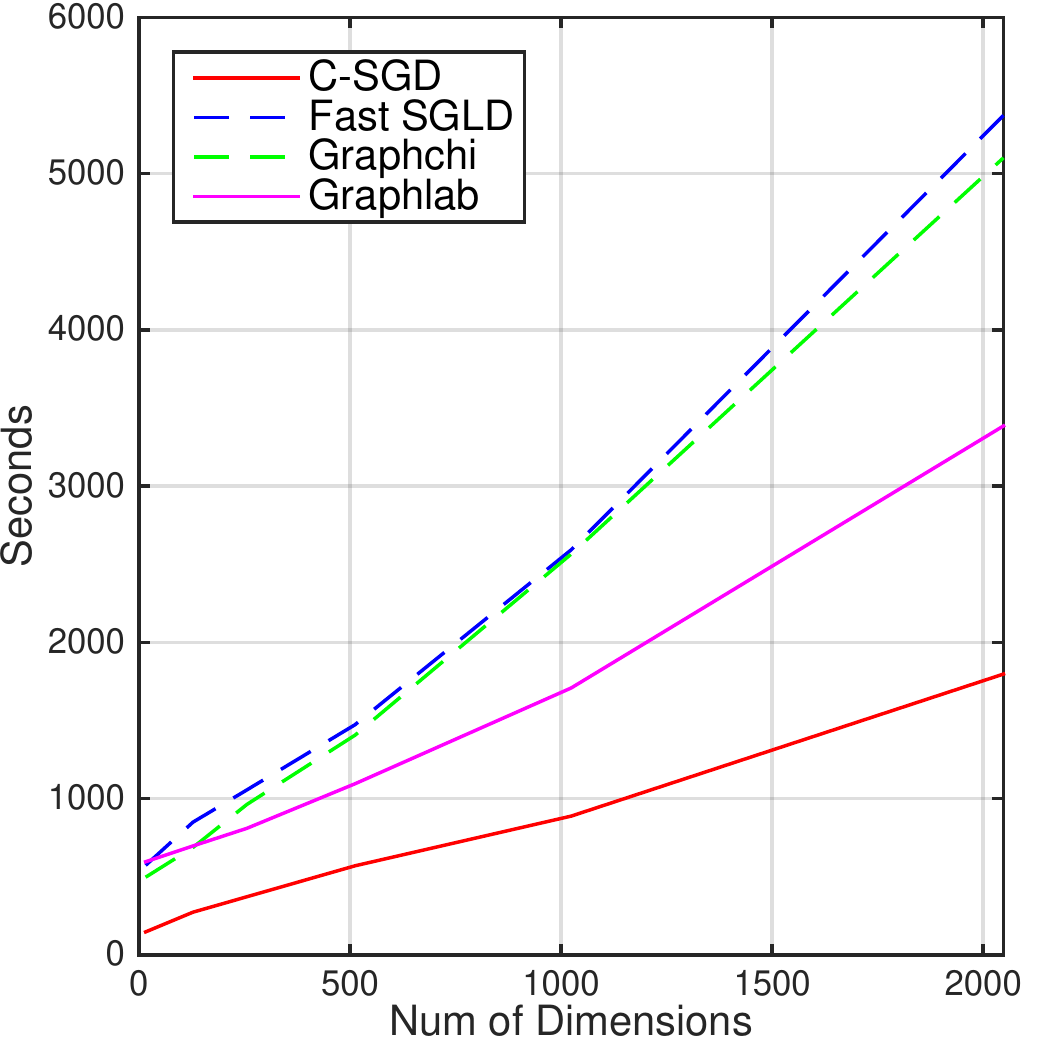}
\caption{ Timing comparisons on Netflix (top, 15 epochs) and Yahoo (bottom, 30 epochs). \label{fig:timing_dimension}}
\end{figure}


\bigskip\smallskip
\subsection{Convergence}\label{sec:exp_conv}
As described above, the convergence of SGLD and SGD
based methods are quite different. We illustrate the convergence on a
small dimension in Figure~\ref{fig:conv}. Basically the C-SGD can find
a MAP estimate using several rounds and then begin overfitting. While
SGLD first needs to burn-in and then start sampling procedure. Note
that SGLD can converge very fast in this case. But for higher
dimensions, SGLD is slower to converge. Careful tuning of the
learning rate is critical here. 

We also investigated the accuracy of the model as a function of the
size of the Gaussian lookup table. That is, we checked whether
replacing explicit access to samples from the Normal distribution by
looking up a consecutive number of precomputed parameters from memory
is valid. As can be seen in
Figure~\ref{fig:scale}, for all but the smallest sets, this
suffices. That is, already once we have more than 10,000 numbers, we
no longer need a Gaussian random number generator and the results
obtained are essentially indistinguishable (obviously for large
numbers of dimensions somewhat more terms are needed). 

\begin{figure}[tbhf]
\centering
\includegraphics[width=\columnwidth]{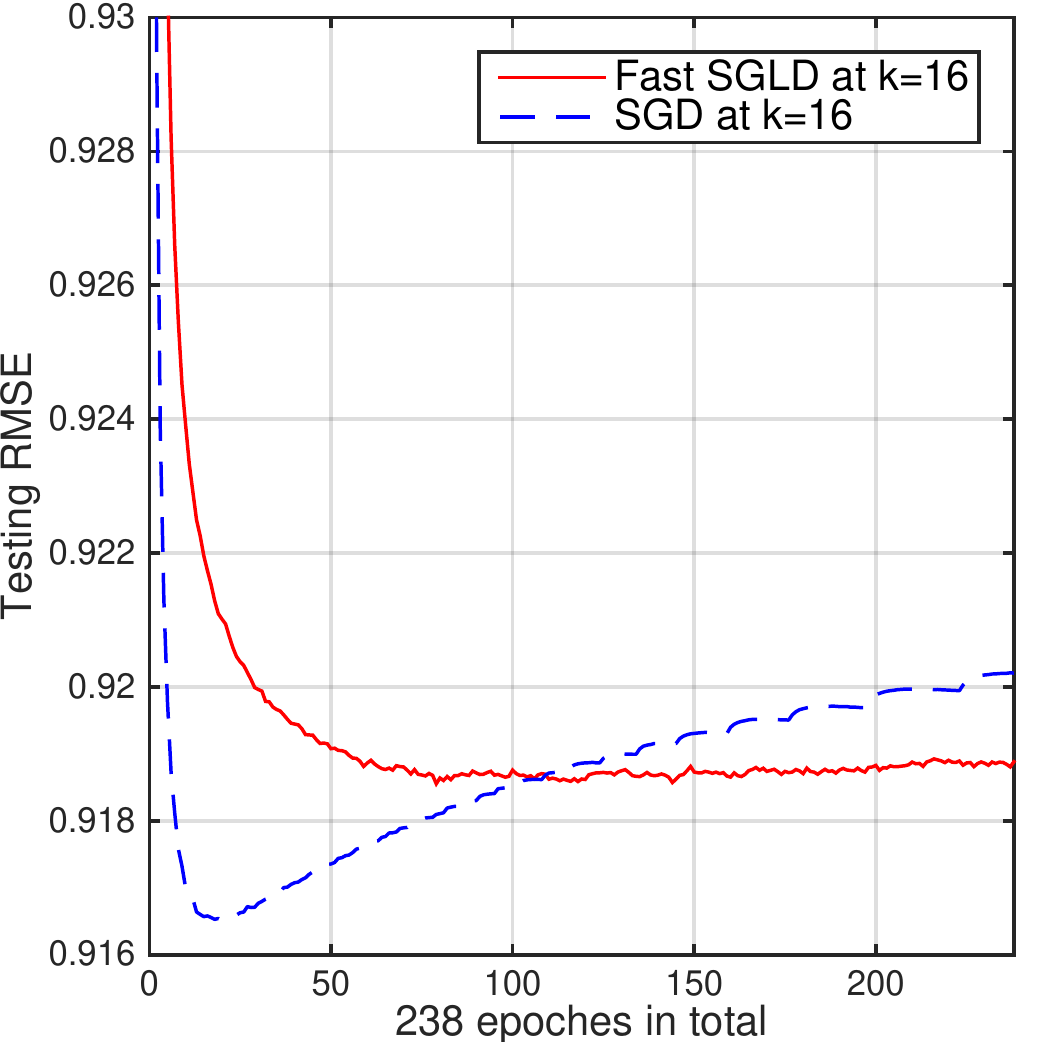}
\caption{ Convergence of SGLD for 16-dimensional models on Netflix. It
  is clear that SGLD does not overfit (although this is not a
  substantial issue for 16 dimensions).
	\label{fig:conv}}
\end{figure}

\subsection{Cache-efficient Design}

\begin{table}
\begin{tabular}{|l|rr|rr|}
\hline
  K & \multicolumn{2}{|c|}{SC-SGD}&
  \multicolumn{2}{|c|}{GraphChi} \\
  & L1 Cache & L3 Cache & L1 Cache & L3 Cache \\ \hline
  16  & 2.84\% & 0.43\% & 12.77\% & 2.21\% \\
  256 & 2.85\% & 0.50\% & 12.89\% & 2.34\% \\
  2048 & 3.3\% & 1.7\% & 15\% & 9.8\% \\ \hline
\end{tabular}
\caption{\label{tb:cache} Cache miss rates in C-SGD and GraphChi. The
  results were obtained using Cachegrind. The cache miss
  rate in GraphChi is considerably higher, which explains to some
  extent the speed difference.}
\end{table}

We show the cache efficiency of C-SGD and Graphchi in this section. Our 
data access pattern can accelerate the hardward cache prefetching. In the meanwhile
we also use software prefetching strategies to prefetch movie factors in advance. But software prefetching 
is usually dangerous in practice while implementing in practice because we need to know the prefetching stride in advance. That is
when to prefetch those movie factors. In our experiments we set prefetching stride to 2 empirically. 
We set the experiments as follows. In each
gradient update step given $r_{ij}$, once the parameters e.g. $u_i$ and
$v_j$ in \eq{eq:sgd} been read they will stay in cache for a while
until they be flushed away by new parameters. What we really care
about in this section is if the first time each parameter be read by
CPU is already staying in cache or not. If it is not in cache then
there will be a cache miss and will push CPU to idle. After that the
succeeding updates (the specific updates depend on the algorithms
e.g. SGD or SGLD) for $u_i$ and $v_j$ will run on cache level. 

We use Cachegrind~\cite{laptevanalysis} as a cache profiler and
analyze cache miss for this purpose. The result in
Table~\ref{tb:cache} shows that our algorithm is quite cache friendly when
compared with GraphChi on all dimensions. This is likely due to the
way GraphChi ingests data: it traverses one data and item block at a time. As a
result it has a less efficient portfolio of access frequency and it
needs to fetch data from memory more frequently. We believe this to be
both the root cause of decreased computational efficiency and slower
convergence in the code.

\subsection{Privacy  and Accuracy}


We now investigate the influence of privacy loss on accuracy. As
discussed previously, a small rescaling factor $B$ can help us
to get a nice bound on the loss function. For private collaborative
filtering purposes, we first trim the training data by setting each user's maximum allowable number of ratings $\tau=100$ and $\tau=200$
for the Netflix competition dataset and Yahoo Music data respectively. We set $B = \tau(5-1+\kappa)^2$ and weight of each user as $w_i = \min(\rho, \frac{B}{m_i(5-1+\kappa)^2})$ where $\kappa$ is set to 1. According to different trimming strength we have $B=2500$ and $B=5000$ for Netflix data and Yahoo data respectively. Note that a maximum
allowable rating from $100$ to $200$ is quite reasonable, since in practice most
users rate quite a bit fewer than $200$ movies (due to the power law
nature of the rating distribution). Moreover, for users who have more
than $200$ ratings, we actually can get a quite a good approximation
of their profiles by only using a reasonable size of random samples of these ratings. As such we get a dataset with 33M ratings for Netflix and 100M ratings for Yahoo Music data. We study the prediction accuracy, i.e. the utility of our private method by varying the differential privacy budget $\epsilon$ for fixed model dimensionality $K=16$. 

The parameters of the experiment are set as follows. For Netflix data, we set $\eta_0=\{6\cdot 10^{-10},3\cdot 10^{-9}, 3.2\cdot 10^{-8}\}$, $\gamma=0.6$, $\zeta = \{7\cdot 10^{-2}, 2.5\cdot 10^{-3}\}$, $\rho = \{1,10\}$. For Yahoo data, we set $\eta_0=\{1.5 \cdot 10^{-10}, 1.5\cdot 10^{-9}, 5\cdot 10^{-10}, 2\cdot 10^{-9}\}$, and $\gamma=\{0.8,0.9\}$, $\zeta=\{0.05, 0.01,0.005\}$, $\rho=\{1,30\}$. In addition, because we are sampling P$(U,V|rest)$ we fix regularizer parameters $\Lambda_u, \Lambda_v$ which are estimated by a non-private SGLD in this section.




While we are sampling $(U,V)$ jointly, we essentially only need to release $V$. 
Users can then apply their own data to get the full model
and have a local recommender system:
\begin{equation}
  \label{eq:semiprivate}
u_i \approx \rbr{\lambda \one + \sum_{j|(i,j) \in \mathcal{S}} v_j v_j^\top}^{-1} \sum_{j} v_j r_{ij}
\end{equation}
The local predictions, i.e.\ in our context the utility of
differentially private matrix factorization method, along the
different privacy loss $\epsilon$ are shown in
Figure~\ref{fig:RMSE_vs_privacy}.

More specifically, the model \eq{eq:semiprivate} is a \emph{two-stage}
procedure which first takes the differentially private \emph{item
  vectors} and then use the latter to obtain locally non-private user
parameter estimates. This is perfectly admissible since users have no
expectation of privacy with regard to their own ratings.

\begin{figure}
\centering
\includegraphics[width=\columnwidth]{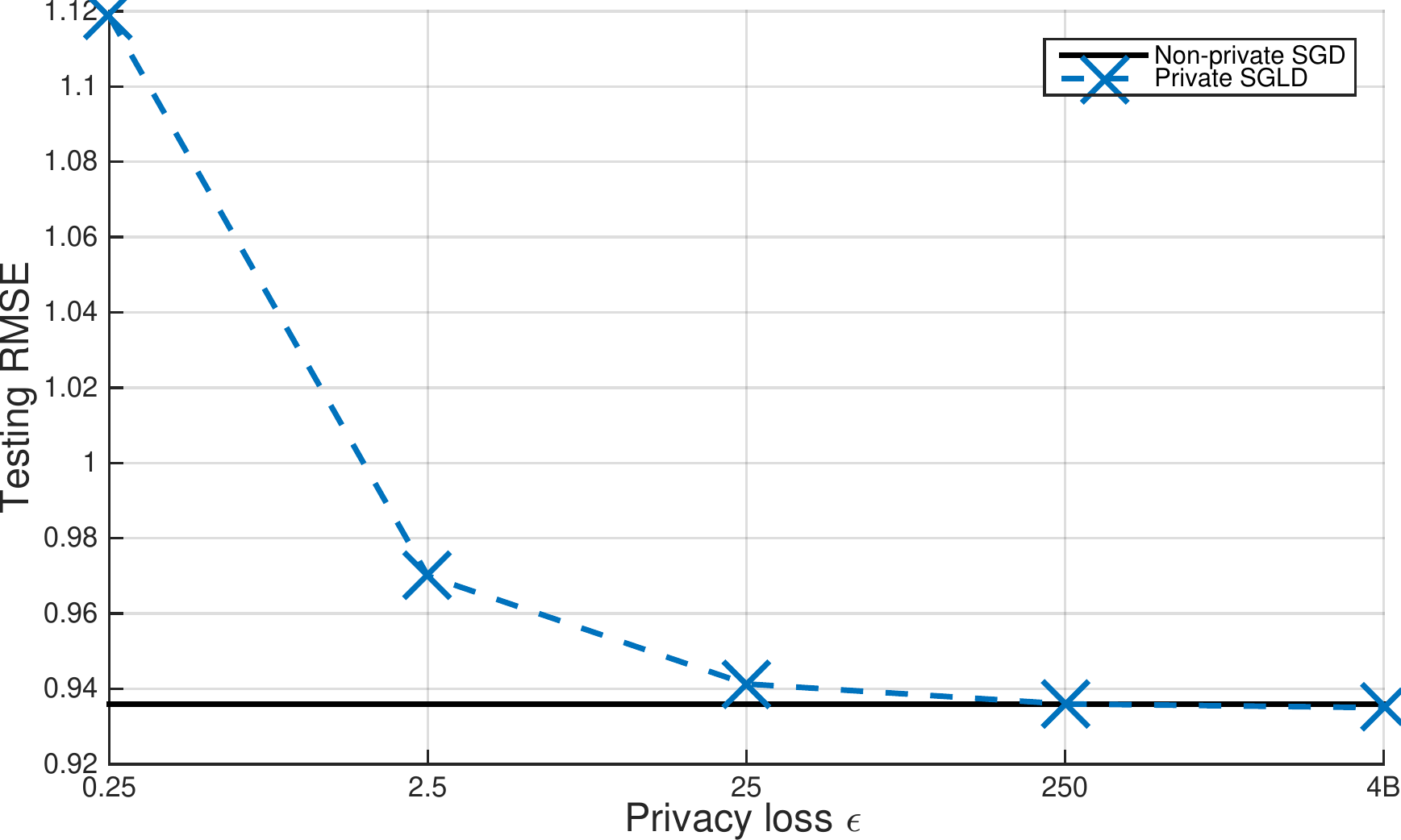}
\includegraphics[width=\columnwidth]{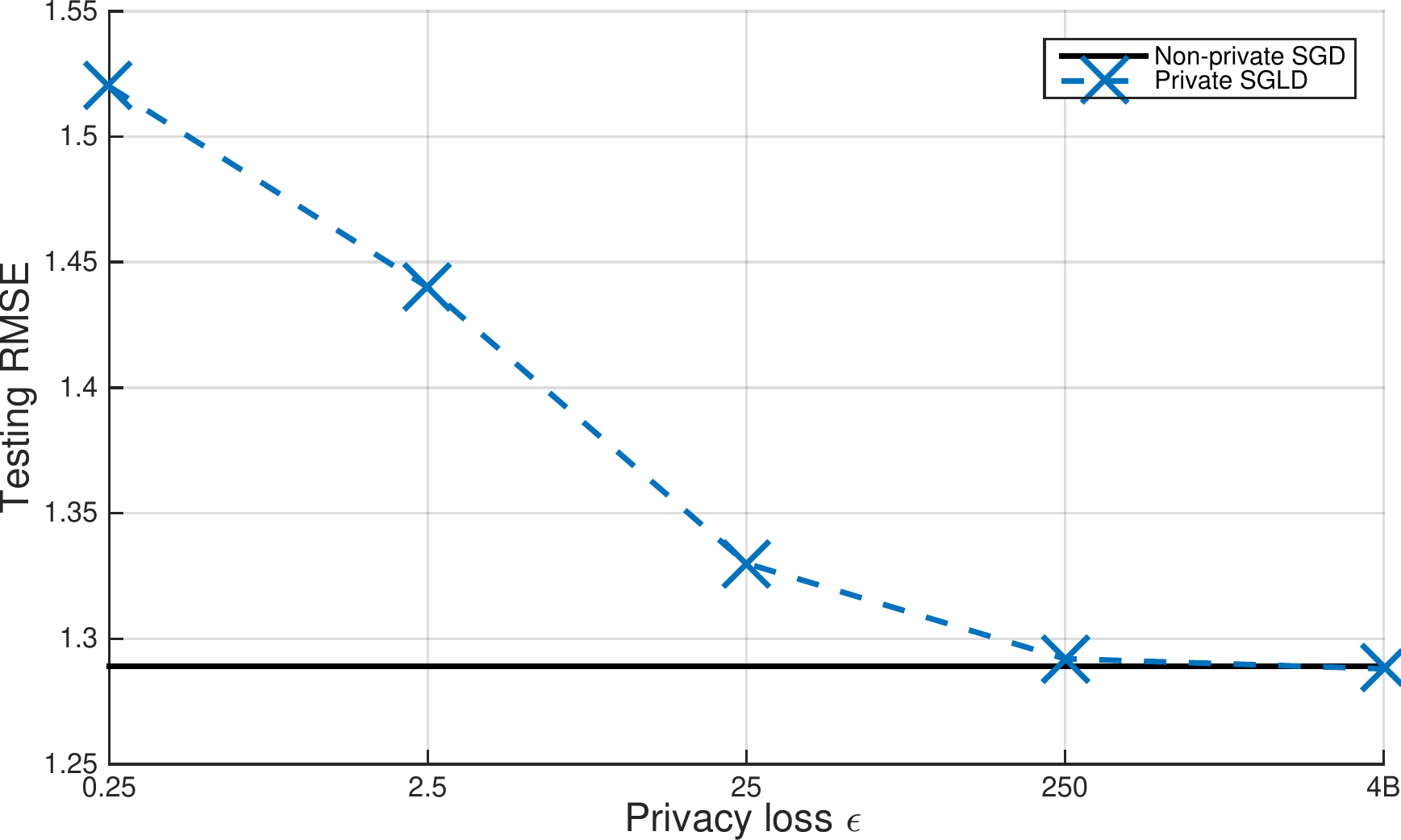}
\caption{Test RMSE vs.\ privacy loss $\epsilon$ on Netflix (top) and
  Yahoo (bottom). A modest decrease in accuracy
  affords a useful gain in privacy. \label{fig:RMSE_vs_privacy}} 
\end{figure}




\subsection{Rating privacy, user privacy and average personalized privacy}
Interpreting the privacy guarantees can be subtle. A privacy loss of $\epsilon=250$ as in Figure~\ref{fig:RMSE_vs_privacy} may seem completely meaningless by Definition~\ref{def:diffp} and the corresponding results in \citet{Mcsherry2009} may appear much better.

We first address the comparison to \citet{Mcsherry2009}. It is important to point out that our privacy loss $\epsilon$ is stated in terms of user level privacy while the results in \citet{Mcsherry2009} are stated in terms of rating level privacy, which offers exponentially weaker protection. $\epsilon$-user differential privacy translates into $\epsilon/\tau$-rating differential privacy. Since $\tau=200$ in our case, our results suggest that we almost lose no accuracy at all while preserving rating differential privacy with $\epsilon<1$. This matches (and slightly improves) \citet{Mcsherry2009}'s carefully engineered system.

On the other hand, we note that the plain privacy loss can be a very deceiving measure of its practical level of protection. Definition~\ref{def:diffp} protects privacy of an arbitrary user, who can be a malicious spammer that rates every movie in a completely opposite fashion as what the learned model would predict. This is a truly paranoid requirement, and arguably not the right one, since we probably should not protect these malicious users to begin with. For an average user, the personalized privacy (Definition~\ref{def:DPpersonal}) guarantee can be much stronger, as the posterior distribution concentrates around models that predict reasonably well for such users. As a result, the log-likelihood associated with these users will be bounded by a much smaller number with high probability. In the example shown in Figure~\ref{fig:RMSE_vs_privacy}, a typical user's personal privacy loss is about $\epsilon/25$, which helps to reduce the essential privacy loss to a meaningful range.

\section{Conclusion}

In this paper we described an algorithm for efficient collaborative
filtering that is compatible with differential privacy. In particular,
we showed that it is possible to accomplish all three goals: accuracy,
speed and privacy without any significant sacrifice on either end.

Moreover, we introduced the notion of \emph{personalized} differential
privacy. That is, we defined (and proved) the notion of obtaining
estimates that respect different degrees of privacy, as required by
individual users. We believe that this notion is highly relevant in
today's information economy where the expectation of privacy may be
tempered by, e.g.\ the cost of the service, the quality of the
hardware (cheap netbooks deployed with Windows 8.1 with Bing), and the
extent to which we want to incorporate the opinions of
users.

Our implementation takes advantage of the caching properties of modern
microprocessors. By careful latency hiding we are able to obtain near
peak performance. In particular, our implementation is approximately 3
times as fast as GraphChi, the next-fastest recommender system. In
sum, this is a strong endorsement of Stochastic Gradient Langevin
Dynamics to obtain differentially private estimates in recommender
systems while still preserving good utility.

{\bfseries Acknowledgments:} Parts of this work were supported by a grant of Adobe Research. Z.\ Liu was supported by Creative Program of Ministry of Education (IRT13035); Foundation for Innovative Research Groups of NNSF of China (61221063); NSF of China (91118005, 91218301); Pillar Program of NST (2012BAH16F02). Y.-X.~Wang was supported by NSF Award BCS-0941518 to CMU Statistics and Singapore National Research Foundation under its International Research Centre @ Singapore Funding Initiative and administered by the IDM Programme Office.

\appendix

\begin{proof}[Proof of Theorem~\ref{thm:privacy}]
	The $\epsilon$-DP claim follows by choosing the utility function to be the $-F(U,V)$ and apply the exponential mechanism \citep{mcsherry2007mechanism} which protects $\epsilon$-DP by output $(U,V)$ with probability proportional to $\exp(-\frac{\epsilon}{2\Delta F(U,V)}F(U,V)$
	Where he sensitivity of function $f$ be defined as
	$$
	\Delta f(X) = \sup_{X,X'\in \cX^n:d(X,X')\leq 1}\|f(X) - f(X')\|_2.
	$$
	All we need to do is to work out the sensitivity for $F(U,V)$ here.
	By the constraint in $U,V$ and $1\leq r_{ij}\leq 5$, we know $(r_{ij}-u_i^Tv_j)^2\leq (5+\kappa)^2$.
	Since one user contributes only one row to the data the trimming/reweighting procedure ensures that
	for any $U,V$ and any user,  the sensitivity of $F(U,V)$ obeys
	$$\Delta F(U,V)\leq 2 w_i\min\{m_i,\tau\} (5+\kappa)^2:=2B,$$
	as specified in the algorithm. The $(\epsilon,\delta)-$DP claim is simple (given in Proposition 3 of \citep{Wang2015}) and we omit here.
	
	Lastly, we note that the ``retry if fail'' procedure will always sample from the the correct distribution of $P$ conditioned on $(U,V)$ satisfying our constraint that $u_i^Tv_i$ is bounded, and it does not affect the relative probability ratio of any measurable event in the support of this conditional distribution.
\end{proof}

\begin{proof}[Proof of Theorem~\ref{thm:personal_privacy}]
	For generality, we assume the parameter vector is $\vct \theta$ and all regularizers is capture in prior $p(\theta)$.
	The posterior distribution $p(\vct \theta|\vct x_1,...,\vct x_n) = \frac{\prod_{i=1}^n p(\vct x_i|\vct \theta)p(\vct \theta)}{\int_{\vct \theta} \prod_{i=1}^n p(\vct x_i|\vct \theta)p(\vct \theta) d\vct \theta}$. For any $\vct x_1,...,\vct x_n$, if we add  (removing has the same proof) a particular user $\vct x'$ whose log-likelihood is uniformly bounded by $B'$.
	The probability ratio can be factorized into
	\begin{align*}
	\frac{p(\vct \theta|\vct x_1,...,\vct x_n, \vct x',)}{p(\vct \theta|\vct x_1,...,\vct x_n)} = \underbrace{\frac{p(\vct x'|\vct \theta)\prod_{i=1}^n p(\vct x_i|\vct \theta)p(\vct \theta)}{\prod_{i=1}^n p(\vct x_i|\vct \theta)p(\vct \theta)}}_\text{Factor 1}\\\times \underbrace{\frac{\int_{\vct \theta} \prod_{i=1}^n p(\vct x_i|\vct \theta)p(\vct \theta) d\vct \theta}{ \int_{\vct \theta} p(\vct x'|\vct \theta)\prod_{i=1}^n p(\vct x_i|\vct \theta)p(\vct \theta)d\vct \theta }}_\text{Factor 2}.
	\end{align*}
	
	It follows that
	\begin{align*}
	&\text{Factor 1} = p(\vct x'|\vct \theta) = e^{\log p(\vct x'|\vct \theta)} \leq e^{B'},\\
	&\text{Factor 2} = \frac{\int_{\vct \theta} \prod_{i=1}^n p(\vct x_i|\vct \theta)p(\vct \theta) d\vct \theta}{ \int_{\vct \theta} p(\vct x'|\vct \theta)\prod_{i=1}^n p(\vct x_i|\vct \theta)p(\vct \theta)d\vct \theta }\\
	=& \frac{\int_{\vct \theta} \prod_{i=1}^n p(\vct x_i|\vct \theta)p(\vct \theta) d\vct \theta}{ \int_{\vct \theta} e^{\log p(\vct x'|\vct \theta)}\prod_{i=1}^n p(\vct x_i|\vct \theta)p(\vct \theta)d\vct \theta }\\
	\leq& \frac{\int_{\vct \theta} \prod_{i=1}^n p(\vct x_i|\vct \theta)p(\vct \theta) d\vct \theta}{ \int_{\vct \theta} e^{-B_i}\prod_{i=1}^n p(\vct x_i|\vct \theta)p(\vct \theta)d\vct \theta }\leq e^{B'}.
	\end{align*}
	As a result, the whole thing is bounded by $e^{2B'}$.
	
	In Algorithm~\ref{alg:DPMF}, denote $\vct \theta = (U,V)$. We are sampling from  a distribution proportional to $e^{\frac{\epsilon}{4B}F(\vct \Theta)}$. This is equivalent to taking the above posterior $p$ to have the log-likelihood of User $\vct x'$ bounded by $\frac{\epsilon 2B'}{4B}=\frac{\epsilon B'}{2B}$, therefore the algorithm obeys $\frac{\epsilon B'}{2B}$ personalized differential privacy for user $\vct x'$. Take $B'$ to be any customized subset of $B_1,...,B_n$ adjustied using $w$ we get the expression as claimed.
\end{proof}

\bibliographystyle{abbrvnat}
\setlength{\bibsep}{0pt plus 0.3ex}


\end{document}